\pdfoutput=1
\documentclass{article}

\usepackage[nonatbib, final]{neurips_2019}

\usepackage[utf8]{inputenc} 
\usepackage[T1]{fontenc}    
\usepackage{hyperref}       
\usepackage{url}            
\usepackage{booktabs}       
\usepackage{subcaption}
\usepackage{amsfonts}       
\usepackage{nicefrac}       
\usepackage{microtype}      
\usepackage{graphicx}
\usepackage{amsmath}
\usepackage{amssymb}
\usepackage{amsthm}
\usepackage{bbm} 
\usepackage{bm} 
\usepackage{verbatim}
\usepackage{xcolor}
\usepackage{xspace}
\usepackage{adjustbox}
\usepackage{multirow}
\usepackage{wrapfig}
\usepackage{enumitem}
\usepackage[ruled,vlined]{algorithm2e}

\SetKwInput{KwInput}{Input}                
\SetKwInput{KwOutput}{Output}  

\newif{\ifhidecomments}
\hidecommentsfalse

\ifhidecomments
	\newcommand{\amit}[1]{}
	\newcommand{\chenhao}[1]{}
	\newcommand{\divyat}[1]{}
\else
    \newcommand{\amit}[1]{\textcolor{red}{Amit: #1}}
    \newcommand{\chenhao}[1]{\textcolor{blue}{CT: #1}}
    \newcommand{\divyat}[1]{\textcolor{brown}{Divyat: #1}}
\fi

\newtheorem{theorem}{Theorem}
\newtheorem{definition}{Definition}[section]

\newcommand{\examplecf}{\texttt{Example-Based CF}\xspace}

\newcommand{\modelcf}{\texttt{Model-based CF}\xspace}
\newcommand{\approxcf}{\texttt{Model-approx CF}\xspace}

\DeclareMathOperator*{\argmin}{argmin}

\newcommand{\vecx}{\bm{x}}
\newcommand{\vecz}{\bm{z}}

\title{Preserving Causal Constraints in Counterfactual Explanations for Machine Learning Classifiers}

\author{%
  Divyat Mahajan \\
  Microsoft Research\\
  Bangalore, India\\
  \texttt{divyatmahajan@gmail.com}
  \And
  Chenhao Tan \\
  University of Colorado Boulder\\
  Boulder, USA\\
  \texttt{chenhao@chenhaot.com} \\
  \And
  Amit Sharma \\
  Microsoft Research\\
  Bangalore, India\\
  \texttt{amshar@microsoft.com}\\
}

\begin{document}

\maketitle
\begin{abstract}
To construct interpretable explanations that are consistent with the original ML model, \emph{counterfactual} examples---showing how the model's output changes with small perturbations to the input---have been proposed. This paper extends the work in counterfactual explanations by addressing the challenge of \emph{feasibility} of such examples. For explanations of ML models in critical domains such as healthcare and finance, counterfactual examples are useful for an end-user only to the extent that perturbation of feature inputs is feasible in the real world.  
We formulate the problem of feasibility as preserving causal relationships among input features and present a method that uses (partial) structural causal models to generate actionable counterfactuals. When feasibility constraints cannot be easily expressed, we consider an alternative mechanism where people can label generated CF examples on feasibility: whether it is feasible to intervene and realize the candidate CF example from the original input. To learn from this labelled feasibility data, we propose a modified variational auto encoder loss for generating CF examples that optimizes for feasibility as people interact with its output. Our experiments on Bayesian networks and the widely used ``Adult-Income'' dataset show that our proposed methods can generate counterfactual explanations that better satisfy feasibility constraints than existing methods. 
Code repository can be accessed here: \textit{https://github.com/divyat09/cf-feasibility}
\end{abstract}

\section{Introduction}
\label{Intro}
Local explanations for a machine learning model are important for people to interpret its output, especially in critical decision-making scenarios such as healthcare, governance, and finance. Techniques for explaining an ML model often involve a simpler surrogate model that yields interpretable information, such as feature importance scores~\cite{lime}. However, these techniques suffer from an inherent fidelity-interpretability tradeoff due to their use of a simpler model for generating explanations. Highly interpretable explanations may end up approximating too much and be inconsistent with the original ML model (low fidelity), while high fidelity explanations may be as complex as the original ML model and thus less interpretable.

\emph{Counterfactual} explanations~\cite{wachter2017counterfactual} have been proposed as an alternative that are always consistent with the original ML model and arguably may also be interpretable. Counterfactual (CF) explanations present the perturbations in the original input features that could have led to a change in the prediction of the model. For example, consider a person whose loan application has been rejected  by an ML classifier. For this person, a CF explanation provides what-if scenarios wherein they would have their loan approved, e.g., \textit{``your loan would have been approved if your income was \$10000 more''}. Since the goal is to generate perturbations of an input that lead to a different outcome from the ML model, CF explanation has parallels with adversarial examples~\cite{goodfellow2014explaining}.

However, one of the biggest challenges with  counterfactual explanation is to generate examples that are \emph{feasible} in the real world. Continuing with the loan example, the counterfactual changes in the input should follow some natural laws (e.g., age cannot decrease) and knowledge about interactions between features (e.g., changing education level without changing age is impossible). 
Recent work~\cite{ustun2019actionable,dhurandhar2018explanations,poyiadzi2020face} tries to address feasibility 
using statistical constraints during generation of CFs, such as encouraging CF examples that are more likely given the training data. 

Our main contribution is to show that feasibility is fundamentally a causal concept, and cannot be addressed with statistical constraints alone. We formally define feasibility for a CF example based on an underlying structural causal model between input features: a CF example is feasible if the changes satisfy constraints entailed by the causal model. To bring out the difference with statistical constraints, consider a CF example that recommends ``increase education level to Masters'' without changing the age of a person. Such a CF example is infeasible but will satisfy constraints from past work, including change in only actionable features~\cite{ustun2019actionable}, and being likely given the observed data~\cite{dhurandhar2018explanations} (it may be likely to observe others with the same age and a Masters degree). Similarly an infeasible CF example that recommends "decrease age by 3 years" satisfies the constraint from \cite{poyiadzi2020face} that intermediate CFs on the path (e.g., people with same features but age reduced by 1 and 2 years) are likely, but is impossible to act upon. We call this \emph{global} feasibility that must be satisfied by all counterfactual examples. 
We also define \emph{local} feasibility that depends on an end-user's context or preferences. 

To generate feasible CF explanations, we propose a \emph{causal proximity} regularizer that can be added to any CF generation method instead of the standard proximity measure based on $\ell_1$ or $\ell_2$ distance~\cite{wachter2017counterfactual}. The proposed proximity loss is based on causal relationships between features, as modeled by a structural causal model (SCM) of 
input features. In practice, we show how the loss can be derived from a partial SCM, or common unary and binary constraints such as monotonic change between features. 
In many cases, however, it is not possible to express feasibility 
with simple constraints. Therefore, we propose a second method that learns feasibility constraints from users' binary feedback on its generated counterfactuals. Our method modifies the standard variational autoencoder's objective to adapt it for generating CFs while encouraging feasibility as defined by users' labels. 



Results on Adult-Income
and synthetic Bayesian network datasets show that our proposed methods can generate counterfactual examples that are more feasible than models that do not include causal assumptions or user feedback. Further, our novel generative model is much faster than existing approaches to generate CFs. To summarize, our contributions include:
\begin{itemize}[itemsep=0pt,leftmargin=*,topsep=-2pt]
    \item We provide a causal view of the feasibility of CF examples that includes many kinds of constraints not considered in prior work. 
    \item To address feasibility, we propose a causal proximity regularizer based on constraints derived from an SCM, that can be applied to any method for generating CF examples.
    \item When feasibility constraints are not available, we propose a VAE-based generative model that can learn feasibility constraints from user feedback. 
\end{itemize}

\section{A Causal View of Feasibility of CF Explanations}
\label{connections}
Throughout, we assume a machine learning classifier, $h:\mathcal{X}\rightarrow \mathcal{Y}$ where $\vecx \in \mathcal{X}$ are the features and $y \in \mathcal{Y}$ is a categorical output.
A \emph{valid} counterfactual example for an input $\vecx$ and outcome $y$ is one that changes the outcome of $h$ to the desired outcome $y'$ \cite{DiverseCf}.
Counterfactual generation is usually framed as solving an optimization problem that searches in the feature space to find perturbations that are proximal (close to the original input) but lead to a different output class from the machine learning model.  \cite{wachter2017counterfactual} provide the following optimization to generate a  CF example $\vecx^{cf}$ for an input instance $\vecx$  given a ML model $h$, where the target class is  $y'$:

\begin{equation}
\label{eq:cf_loss}
    \argmin_{\vecx^{cf}} \operatorname{Loss} (h(\vecx^{cf}), y') + \operatorname{Dist}(\vecx, \vecx^{cf}).
\end{equation}

Loss refers to a classification loss (such as cross-entropy) and $\operatorname{Dist}$ refers to a distance metric (such as  $\ell_2$ distance).
That is, we seek to generate counterfactual explanations that belong to a target class $y'$ while still remaining proximal to the original input. 

Note that under this formulation, there are two limitations: 1) features of the input $\vecx$ may be changed independently to construct $\vecx^{cf}$, and 2) a new optimization problem needs to be solved for each new input. For the first, we propose a definition for incorporating feasibility below. For the second, Section~\ref{CFVAE} provides a generative model that once trained, can easily generate multiple new CFs.

\subsection{Global Feasibility}
\begin{definition}
\textbf{Causal Model~\cite{pearl2009book}.} A causal model is a triplet $M=\langle U, V, F \rangle$ such that $U$ is a set of exogenous variables,
$V$ is a set of endogenous variables that are determined by variables inside the model, 
and  $F$ is a set of functions that determine the value of each $v_i \in V$ (up to some independent noise) based on values of $U_i \bigcup Pa_i$ where $U_i \subseteq U$ and $Pa_i \subseteq V\setminus v_i$. 
\end{definition}

\begin{definition}
\textbf{Global Feasibility.} Let $\langle \vecx_i,y_i\rangle$ be the input features and the predicted outcome from $h$, and let $y'$ be the desired output class. Let $M=\langle U,V, F \rangle$ be a causal model over $\mathcal{X}$ such that each feature is in $U \bigcup V$. 
Then, a counterfactual example $\langle \vecx^{\tt cf},y^{\tt cf} \rangle$ is globally feasible if it is valid ($y^{\tt cf}=y'$),  the change from $\vecx_i$ to $\vecx^{cf}$ satisfies all constraints entailed by the causal model, and all exogenous variables $\vecx^{exog}=U$ lie within the input domain.
\end{definition}
For example, a CF example that changes an individual's age to $300$ is infeasible since it violates the limits of the input domain of the age feature.  In general, such constraints relating to the input domain may be learned from an i.i.d. sample of data by estimating the joint distribution of features. E.g., \cite{dhurandhar2018explanations} use an auto-encoder 
loss-term 
to align CF examples to the data distribution.

In addition, however, a CF example that decreases age is infeasible since it violates the natural causal model/constraint that age can only increase with time. Such {\em causal constraints cannot be learned from data alone, and often need extra information}~\cite{pearl2009book}. 
More complex causal constraints can be defined over pairs or multiple variables. 
For example, in the loan decision example from above, we can consider $v_{p1}$ as {\em education-level} and $v$ as {\em age}, and posit a causal relationship 
that increasing {\em education-level} needs years to complete and thus causes {\em age} to increase.
That is, any counterfactual example that increases {\em education-level} without increasing {\em age} is infeasible,
although a counterfactual example that increases {\em age} without changing {\em education-level} may still be feasible as we do not know the full set of causes that may increase {\em age}. While some of these feasibility constraints can be formulated in simple terms, causal relationships over multiple features can lead to complex constraints.  As we show below, they can be defined formally using structural causal models, and implemented as constraints on how features can change (Section~\ref{sec:model-approx}).  
Constraints from an SCM can also
be probabilistic (e.g., increasing education level with a six-month increase in age is unlikely). 
We may define \emph{degree of feasibility} as the joint probability of 
changes in selected features.   

\subsection{Local Feasibility}
For a particular user, a globally feasible CF example may still be infeasible due to an end-user's context or personal preferences. Thus, while global feasibility is a necessary condition for feasible CFs, we also need to define local feasibility.
\begin{definition}
\textbf{Local Feasibility}: A CF example is locally feasible for a user  if it is globally feasible and satisfies user-level constraints. 
\end{definition}

For example, a user may find it difficult to change their city because of family constraints. Thus, a counterfactual example may be locally infeasible due to many user specific factors. Preserving global constraints entailed from a causal model provides necessary conditions for a feasible counterfactual, but customization may be needed for local feasibility. 

\subsection{A causal proximity loss for generating CF examples}
\label{Causal-CFVAE}
We now define a feasibility-compatible notion of distance to constrain independent perturbations of features.
We want the counterfactual to be proximal to the data sample not only based on the Euclidean distance between them, but also based on the causal relationships between features. 

Suppose we are provided with the structural causal model~\cite{pearl2009book} for the observed data, including the causal graph $G$ over $U \bigcup V$ 
and the functional relationships between variables. $V$ is the set of all endogenous nodes that have at least one parent in the graph. For exogenous variables $U$, we use the standard proximity loss (e.g., using the $\ell_1$/$\ell_2$ distance). 
For each endogenous node $v \in V$ that has at least one parent in the graph, we propose a new feasibility-compatible distance metric based on the generating mechanism of $v$ conditioned on its parents, i.e., $v = f(v_{p1},..,v_{pk}) + \epsilon$ where $v_p$ refers to parent nodes of $v$ and $\epsilon$ denotes independent random noise. 
For each node $v \in V$, 

\begin{equation*}
    \operatorname{DistCausal_v}(x_{v}, x_{v}^{cf}) = 
    \operatorname{Dist_v}(x_{v}^{cf}, f(x_{v_{p1}}^{cf},..,x_{v_{pk}}^{cf}))
\end{equation*}

\begin{wrapfigure}{r}{0.4\linewidth}
    \includegraphics[width=\linewidth]{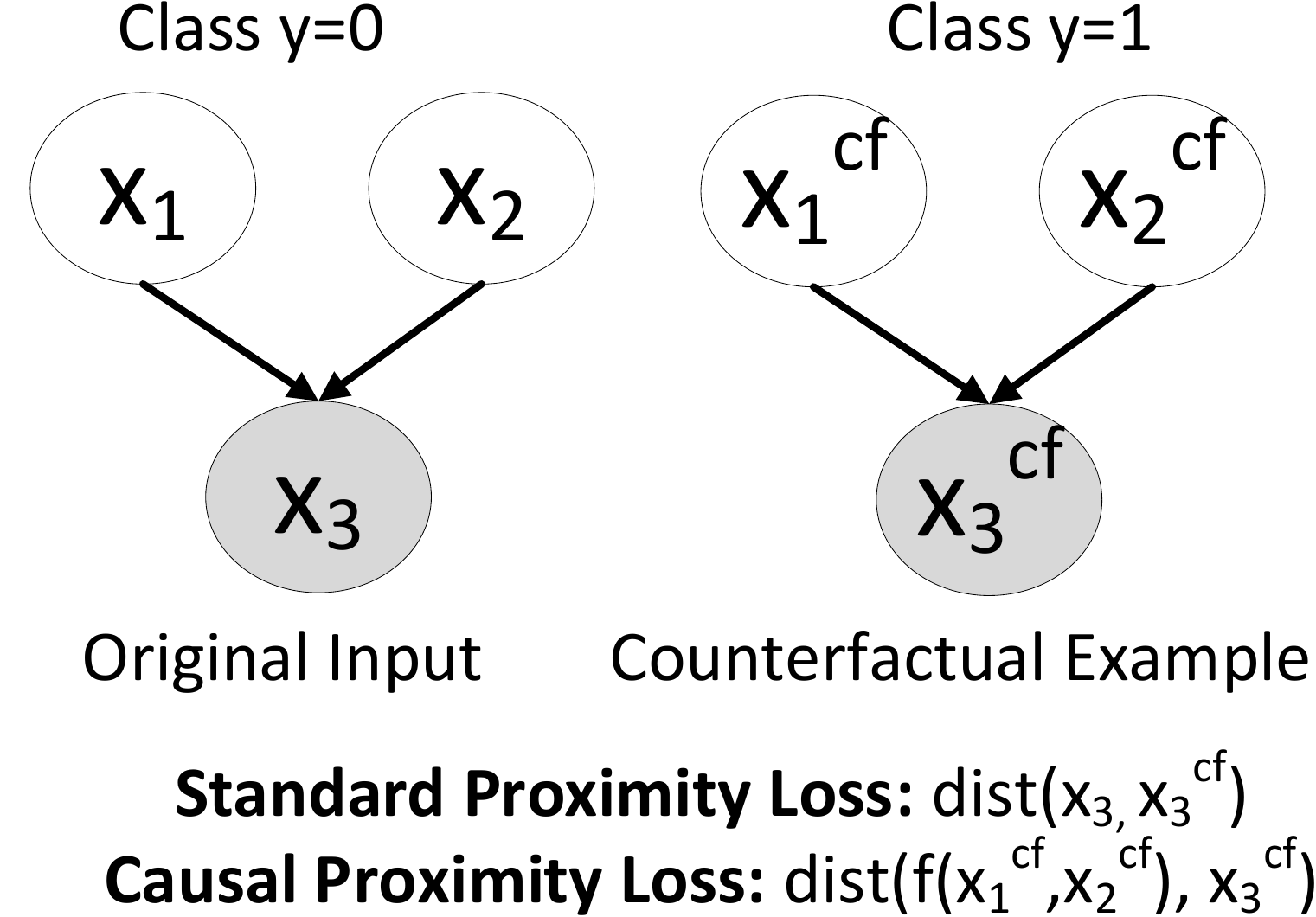}
    \caption{Defining the proximity loss with SCM.}
    \label{fig:causal-proximity}
\end{wrapfigure}
where $f(x_{v_{p1}}^{cf},..,x_{v_{pk}}^{cf})=\mathbb{E}[x_{v}^{cf} | x_{v_{p1}}^{cf},..,x_{v_{pk}}^{cf}]$.
This distance metric (Fig.~\ref{fig:causal-proximity}) indicates that the counterfactual value for the feature $v$ should depend on the values of its parents in the counterfactual example. Once its parents' counterfactual values have been decided, its value should ideally be the one predicted by the SCM function $f$. 
Note that the distance metric does not 
compare to the original value $x_v$, as in the standard proximity loss from Equation~\ref{eq:cf_loss}. Having the standard proximity loss on the exogenous features and the causal proximity loss on all other features ensures that an ideal CF is close to the original input instance and also preserves the causal relationship between features.

Hence, the $\operatorname{Dist}$ term in the CF generation loss function (Eq.~\ref{eq:cf_loss})  can be modified to generate counterfactuals that preserve causal constraints, where $U$ are the exogenous nodes (i.e., nodes without any parents in the causal graph) and $V$ are the remaining features.
\begin{equation*} 
\label{CausalProximity}
\begin{split}
\operatorname{DistCausal}(\vecx, \vecx^{cf}) = \sum_{u \in U}^{}  \operatorname{Dist_u} (x_{u}^{cf}, x_{u}) + \sum_{v \in V}^{} \operatorname{DistCausal_v}(x_{v}, x_{v}^{cf}) 
\end{split}
\end{equation*}
Since knowing the full causal graph is often impractical, the above approach can also work whenever we have partial knowledge of the causal structure (e.g., some edges in the causal graph).
From this partial causal knowledge, we construct a set of nodes $V$ for which we know the generating mechanism for each node $v \in V$ conditioned on its parents and consider the rest of the variables in $U$.
We refer to this approach as {\bf \modelcf}.

\subsection{Building a feasibility-compatible CF explanation method}
\label{sec:model-approx}
Our proposed proximity loss can be combined with any prior CF generation method, by replacing the proximity term:
$\argmin_{\vecx^{cf}} \operatorname{Loss} (h(\vecx^{cf}), y') + \operatorname{DistCausal}(\vecx, \vecx^{cf})$.
However, in practice, the exact functional causal mechanism for a variable is often unknown. Therefore,  we present a simple approximation of the above loss by directly optimizing for certain constraints based on domain knowledge. For example, one  may know that Age of a person cannot decrease, or that Education-level shares a monotonic causal relationship with Age, without knowing the true functional form. 
Below we provide example loss terms for common unary and binary feasibility constraints. 

\textbf{Unary constraints.}
We consider unary constraints that stipulate whether a feature can increase or decrease
and define a hinge loss on the feature of interest. As an example, for the case that a feature can only increase, the hinge loss would be as follows:  $   -\min(0, \vecx^{cf}_v - \vecx_v)$.

\textbf{Binary constraints.}
Binary constraints capture the nature of causal relationship between two features. One of the most common are monotonic constraints, which we approximate by learning an appropriate linear model for each binary constraint. Let $x_1$ and $x_2$ be two features where $x_1$ causes $x_2$ and we have a monotonically increasing trend between them. We capture this monotonic trend by learning a linear model between $x_1$ and $x_2$, under the constraint that the parameter that relates $x_1$ to $x_2$ should be positive (or negative depending on the nature of monotonicity) . This can be learnt by minimizing the following loss function over training data:
 $+ (\vecx_{v_2} - \alpha - \beta \vecx_{v_1}) - \min(0, \beta)$, 
    where $\alpha$ and $\beta$ are parameters that can be learned from training data. We refer to this approach as {\bf \approxcf}. 

\label{sec:gen-cf}

\section{Example-Based Generation of Feasible CF explanations}
In most situations, however, it is difficult to express complex constraints involving multiple features as convex loss terms. A more practical setting is that a user may provide feedback on generated counterfactuals to say which ones are feasible, and the CF generation method should learn feasibility constraints through this feedback. For this to work interactively with a user, we need a pre-trained CF generation method that can generate initial CFs for different inputs and then update its generation method in a fast manner. Past approaches~\cite{dhurandhar2018explanations, DiverseCf}  to CF generation run a new optimization for each input and 
can be difficult to fine-tune 
online.  We thus propose a model-based method,
{\bf \examplecf}. It includes a variational autoencoder (VAE) module that parameterizes generation of CFs, and a fine-tuning model that updates model parameters to support feasibility. 

\subsection{Base VAE generator for CF explanations}
\label{CFVAE}
We define the CF objective as generating CF examples $\vecx^{cf}$ 
as building a model that maximizes $\Pr(\vecx^{cf}| y', \vecx)$ such that $\vecx^{cf}$ belongs to class $y'$.  Our approach is based on an encoder-decoder framework where the task of the encoder is to project input features to a suitable latent space and the task of the decoder is to generate a counterfactual from the latent representation given by the encoder. Analogous to a variational auto-encoder (VAE)~\cite{vae}, we first arrive at a latent representation $\vecz$ for the input instance $\vecx$ via the encoder $q(\vecz|\vecx, y')$ and then generate the corresponding counterfactual $\vecx^{cf}$ via the decoder $p(\vecx^{cf}|\vecz, y')$. Following the construction in VAEs~\cite{vae}, we first derive the evidence lower bound (ELBO) for generating CF explanations. \\

\begin{theorem}
The evidence lower bound to optimize the CF objective $\Pr(\vecx^{cf}| y', \vecx)$ is:
\begin{equation*}
\begin{split}
{ \ln \Pr(\vecx^{cf}| y', \vecx) \geq  \mathbb{E}_{Q(\vecz|\vecx, y')} \ln P( \vecx^{cf} | \vecz, y', \vecx) 
- \mathbb{KL}(Q(\vecz|\vecx, y') ||P(\vecz| y', \vecx) }.
\end{split}
\end{equation*}
\end{theorem}

The proof is in the \textbf{Suppl. A.} 
The prior of the latent variable $\vecz$ is modulated by $y'$ and $\vecx$, but following \cite{sohn2015cvae}, we simply use 
$p(\vecz | y', \vecx) 
\sim \mathcal{N}(\mu_{y'}, \sigma^2_{y'})$,
so the KL Divergence can be computed in closed form.  $P(\vecx^{cf} | \vecz, y', \vecx)$ represents the probability of the output $\vecx^{cf}$ given the desired class and 
latent variable $\vecz$. 
This can be empirically estimated by the $\ell_1$/$\ell_2$ loss or any general $\operatorname{Distance}$ metric between input $\vecx$ and $\vecx^{cf}$. That is, without additional assumptions, we are assuming that probability $P(\vecx^{cf})$ is highest near $\vecx$. In addition, this probability expression is conditioned on $y'$, implying that $\vecx^{cf}$ is valid only if belongs to $y'$ class when applied with $h$. We thus use a classification loss (e.g., hinge-loss) between $h(\vecx^{cf})$ and $y'$, where $y'$ represents the target class and $\beta$ represents the margin where $\lambda$ is a hyperparameter. 

\begin{equation*}
\small
\mathbb{E}_{Q(\vecz|\vecx, y')} \ln P( \vecx^{cf} | \vecz, y', \vecx) 
 \approx  \mathbb{E}_{Q(\vecz|\vecx, y')} [\operatorname{Dist}(\vecx, \vecx^{cf}) + \lambda \operatorname{HingeLoss}(h(\vecx^{cf}), y', \beta )],
\end{equation*}

where the hinge loss function is defined over the softmax scores $s_y$ for each class output of $h$:
{
\small
$ HingeLoss( h(\vecx^{cf}), y', \beta )= max\{ [ \max_{y!=y'}\{s_y(\vecx^{cf})\} - s_{y'}(\vecx^{cf})], -\beta \} $
    }.
The above Hinge Loss formulation encourages classifier's score on target class to be higher than any other class by at least a margin of $\beta$. 
 To summarize, given the ML model $h$ to be explained, we learn our proposed model by minimizing the following loss function (\texttt{BaseGenCFLoss}):
 
\begin{equation*} \label{GenerativeLoss}
    \small \mathop{\mathbb{E}}_{Q(\vecz|\vecx, y')}[%
    \operatorname{Dist}(\vecx, \vecx^{cf}) 
    + \lambda \operatorname{HingeLoss}(h(\vecx^{cf}), y', \beta) ]
     + KL(Q(\vecz|\vecx,y')||P(\vecz|y', \vecx) ),
\end{equation*}    

where $y'$ is the target counterfactual class. Our loss formulation bears an intuitive resemblance with the standard counterfactual loss formulation (Eq. \ref{eq:cf_loss}).  $\operatorname{HingeLoss}(h(\vecx^{cf}), y', \beta)$ helps us to generate valid counterfactuals with respect to the ML model $h$, and $\operatorname{Distance}(\vecx, \vecx^{cf}$) helps us to generate counterfactuals that are close to the input feature. The additional third term in the loss function represents the KL divergence between the prior distribution $p(\vecz|y')$ and the latent space encoder $q(\vecz|\vecx, y')$, analogous to the loss term in a VAE~\cite{vae}. Our encoder-decoder framework can be viewed as an adaptation of VAE for the task of generating counterfactuals. 
 Typically, the $\operatorname{Dist}$ function can be defined as the $\ell_1$ distance between the input $\vecx$ and the counterfactual $\vecx^{cf}$: $\operatorname{Dist}(\vecx, \vecx^{cf}) = \|\vecx-\vecx^{cf}\|_1$. 




\subsection{Learning feasibility constraints through user feedback}
\label{Oracle-CFVAE}

We consider the user as an Oracle that provides binary yes/no feedback on feasibility of a generated CF. Given any input pair   $(\vecx, \vecx^{cf})$ the oracle outputs $1$ if the CF example is feasible, otherwise it outputs $0$.
Hence, in order to generate feasible counterfactuals, our task is to maximize the Oracle score with as few queries to the Oracle as possible. Consider a dataset $(\vecx_i, \vecx'_i, o_i)^q_{i=1}$ where $\vecx'_i$ is the CF example for $\vecx_i$ and $o_i$ is the output of the oracle.

For any new $\vecx$, ideally the output $\vecx^{cf}$ of our model should have a high score. Thus the $\vecx^{cf}$ should be similar to the feasible CFs and dissimilar to infeasible CFs in the query set. We write the similarity of $\vecx^{cf}$ generated by our model and a query $(\vecx_i, \vecx'_i)$ as: 
$ \operatorname{sim}(\vecx_i^{cf}, \vecx'_i)= \exp(-(\vecx'_i - \vecx_i^{cf})^T (\vecx'_i -\vecx_i^{cf}))$
where $\vecx_i^{cf}$ is the output of the VAE with input $\vecx_i$. The similarity should be higher when $o_i=1$ and lower when $o_i=0$, leading to the loss: $\sum_{i=1}^q ||o_i - \operatorname{sim}(\vecx_i^{cf}, \vecx'_i)||_2^2$.

Thus the proposed algorithm has two phases:


\begin{algorithm}[t]
\small
    \SetAlgoLined
    \KwInput{Training data $(\vecx,y)^n_{j=1}$}
    \KwOutput{Counterfactuals $\vecx^{cf}$}
     Base training phase: Learn a base VAE and generate query CFs $(\vecx_i, \vecx'_i)$ to be labelled for feasibility. \\
  Feasibility learning phase: Given labelled queries $(\vecx_i, \vecx'_i, o_i)_{i=1}^q$, fine-tune the trained VAE with the following loss where $\lambda$ is a hyperparameter trading off between validity/proximity and feasibility: 
\begin{align*}
\begin{split}
   \min \sum_{i=1}^{q}     [\operatorname{BaseGenCFLoss}(\vecx_i, \vecx_i^{cf}) 
    +
   \lambda_o
 ||o_i - \operatorname{sim}(\vecx_i^{cf}, \vecx'_i)||_2^2.
\end{split}
\end{align*}
    \caption{\examplecf}
\end{algorithm}

Here is an example of how the method can be used to enforce positive  individual treatment effects (ITE) among features. Consider an example of feature $v$ and its causes $(v_{1}, \ldots, v_{p})$, with a positive ITE of each cause on the feature $v$. This can be captured implicitly by a simple Oracle O. 

\begin{equation*}
\label{eq:simple-oracle}
     O(\vecx, \vecx^{cf})=
    \begin{cases}
      1, &  \text{if}\ (\ \forall \text{i}\ \{\vecx_{v_{pi}}^{cf} > \vecx_{v_{pi}}\} \implies {\vecx_v^{cf}>\vecx_v} )\ 
      \text{or}\ \\ & (\ \forall \text{i}\ \{\vecx_{v_{pi}}^{cf} < \vecx_{v_{pi}}\} \implies \vecx_v^{cf}< \vecx_v ) \\
      0, & \text{otherwise}
    \end{cases}
\end{equation*}
 
Additionally, the method can be used to capture personalized user-specific constraints with a user as the oracle,
thus representing human perception.  Different users could be modeled using different oracles. The oracle can give labelled data for (multiple) user-specific constraints by simply stating a counterfactual as feasible ($O(\vecx, \vecx')=1$) or infeasible ($O(\vecx, \vecx')=0$ ). 

\section{Empirical Evaluation}\label{sec:eval}
We evaluate our proposed methods,  \texttt{\modelcf}, \texttt{\approxcf}, 
and \texttt{\examplecf} on the Adult dataset and simulated Bayesian network datasets. For a fair comparison to \texttt{\examplecf}, we use its base bariational autoencoder as the base CF generator for both  \texttt{\modelcf} and \texttt{\approxcf}.
As we have mentioned before, a counterfactual could be infeasible due to many reasons, but for our evaluation, we assume that there is a constraint that completely captures the feasibility of a counterfactual. To design this constraint, we either infer it from the causal model (simulated datasets) or from domain knowledge (real-world dataset).

\subsection{Datasets.}
\textbf{Simple-BN.} We consider a toy dataset of 10,000 samples with three features $(\vecx_1, \vecx_2, \vecx_3)$ and one outcome variable ($y$). The causal relationships between them are modeled as:
{
\begin{align*}
    p(x_1)& \sim N(\mu_{1},\sigma_{1}); \quad \quad p(\vecx_2) \sim N(\mu_{2},\sigma_{2}) \\
     p(x_3|x_1, x_2 )&  \sim N(k_{1}*(x_1+x_2)^{2} + b_{1}, \sigma_{3}); k_{1} >0, b_{1}>0; \\
     p(y|x_1, x_2, x_3 )& \sim Bernoulli(\sigma( k_{2}*(x_1*x_2) + b_{2} - x_3) );  k_{2} >0, b_{2}>0
 \end{align*}
 }%
Note that we require $x_1$ and $x_2$ to be positive, so they follow a truncated normal distribution.
The intuition is that $\vecx_3$ is determined by a monotonically increasing function of $\vecx_1$ and $\vecx_2$. At the same time,  $y$ is positively affected by an increase in $\vecx_1$ and $\vecx_2$ but negatively affected by $\vecx_3$. Thus, a naive counterfactual method may not satisfy the monotonic constraint on $(\vecx_1, \vecx_2)$ and $\vecx_3$. Specifically, the global monotonicity constraint is defined as: ``{\small ($\vecx_1,\vecx_2$ increase $\implies$ $\vecx_3$ increases) AND ($\vecx_1,\vecx_2$ decrease $\implies$ $\vecx_3$ decrease)}''.

We report results for a hand-crafted set of parameters such that it is possible to have a  tradeoff between proximity and monotonic feasibility constraint: {\small $\mu_1=50$, $\mu_2=50$, $\sigma_1=15$ $\sigma_2=17$, $\sigma_3=0.5$, $k_1=0.0003$, $k_2=0.0013$, $b_1=10$, $b_2=10$}.\\

\textbf{Sangiovese~\cite{magrini2017sangiovese}.} This is a conditional linear Bayesian network on the effects of different agronomic settings on quality of Sangiovese grapes~\cite{bnlearn-repo-url}. It has 14 features and a categorical output for quality, with a sample size of 10,000. The true causal model is known. The features are all continuous except Treatment which has 16 levels. For simplicity, we remove the categorical variable Treatment since it leads to 16 different linear functions.  For feasibility, we test a monotonic constraint over two variables, $BunchN$ and $SproutN$. Specifically, the global monotonicity constraint is defined as: 

($SproutN$ increase $\implies$ $BunchN$ increases) AND ($SproutN$ decrease $\implies$ $BunchN$ decrease) \\

\textbf{Adult~\cite{adult-dataset}.}
We consider a real-world dataset, \texttt{Adult}.
The outcome $y$ is binary $y=0$ (Low Income), and $y=1$ (High Income). Since we do not have a causal model, we design two constraints that capture feasibility using  domain knowledge:  

$\text{C1: } \vecx_{Age}^{cf} \geq \vecx_{Age}$ 

$\text{C2: } (\vecx_{Ed}^{cf} > \vecx_{Ed} \implies \vecx_{Age}^{cf} > \vecx_{Age} )   
\text{ AND } (\vecx_{Ed}^{cf} = \vecx_{Ed} \implies \vecx_{Age}^{cf} \geq \vecx_{Age} )$

$\tt C1$ represents a unary constraint that Age cannot decrease in CF explanations. $\tt C2$ represents a monotonic constraint that increase in Educational level should increase Age, and if Educational level remains the same, age should not decrease. $\tt C2$ also includes an additional constraint that Education level cannot decrease. Hence, if Education level decreases then its an infeasible counterfactual, regardless of the change in Age (since in practice Education level does not usually decrease). To make the CF generation task more challenging, we sample data points with the Age feature greater than 35 and outcome class $y=0$ and data points with the value of feature Age less than 45 and the outcome class $y=1$. This creates a setup in which higher age data points are more correlated with the low income class group and vice-versa. We obtain a dataset of size $15691$ and consider the task of generating CFs with the target class as $y=1$.

\subsection{Evaluation Setup}

\begin{wrapfigure}{r}{0.5\textwidth}
\centering
\begin{tabular}{ccc}
\includegraphics[width=0.24\textwidth]{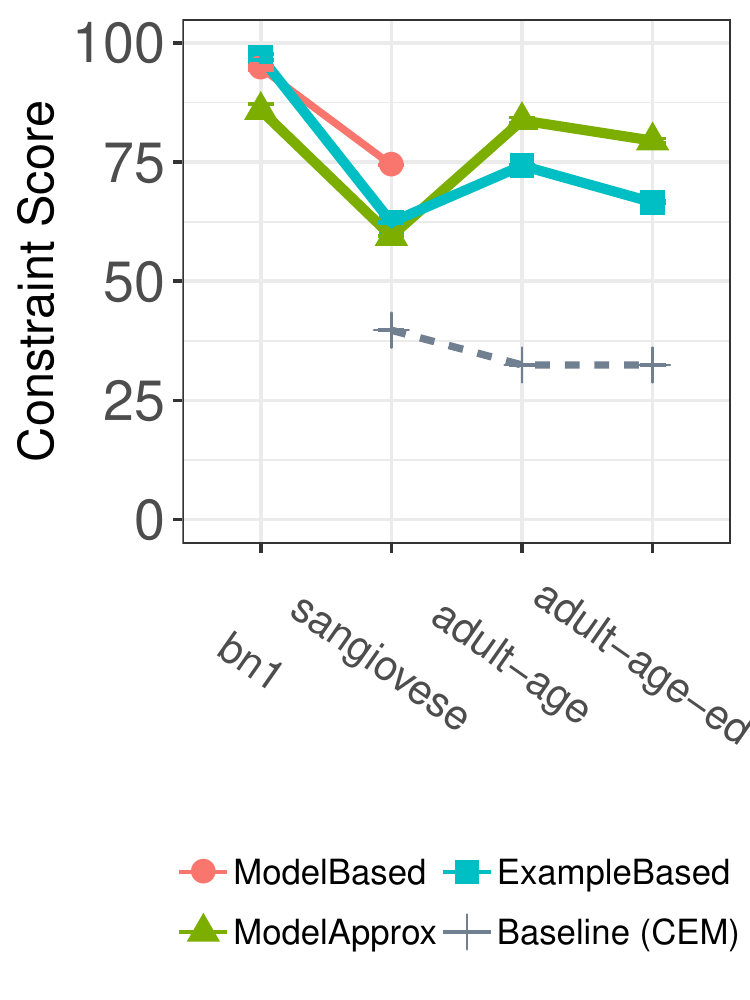} &
\includegraphics[width=0.24\textwidth]{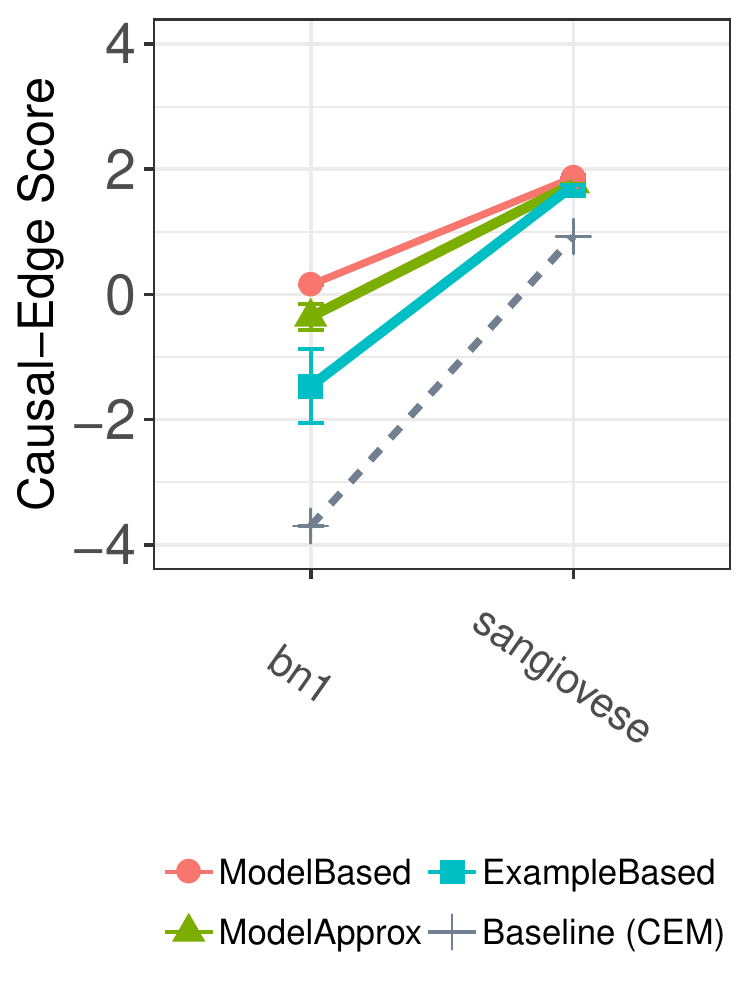} 
\\
(a) & (b) 
\end{tabular}
\caption{Constraint-Feasibility for three datasets, Causal-Edge score for BN1 and Sangiovese.}
\label{fig:main-metrics}
\end{wrapfigure}
For all  experiments, the ML classifier $h$ is implemented as a neural network with two hidden layers, with non-linear activation (ReLU) on the first hidden layer. Continuous features are scaled to (0-1) range and categorical features are represented as one-hot encoded vectors. Each proposed method for counterfactual generation is trained using a  80-10-10\% training, validation and test dataset respectively. For the \examplecf method, we additionally  generate the query set $Q$ using 10\% of the training dataset with 10 counterfactuals per data point. Details regarding the ML model, the base VAE architecture, and hyperparameter tuning for all methods are provided in Suppl. \textbf{C.1}. \\

\noindent \textbf{Methods.} We compare results for \modelcf, \approxcf, and \examplecf with
   CEM, the state-of-the-art contrastive explanations method proposed by~\cite{dhurandhar2018explanations} that uses an auto-encoder to model probability distribution of train data. 
Suppl. \textbf{C.1} describes the loss terms used for modeling constraints in \approxcf for different datasets. \\

\textbf{Evaluation Metrics.} We define the following metrics to evaluate CF examples. 
\begin{itemize}[itemsep=0pt,leftmargin=*,topsep=-2pt]
    \item      \texttt{Target-Class Validity}: \% of CFs whose predicted class 
    is 
    the target class;  
    \item  \texttt{Cont-Proximity}: Proximity for continuous features as the average $\ell_1$-distance between $\vecx^{cf}$ and $\vecx$ in units of median absolute deviation for each features~\cite{DiverseCf}; 
    \item     \texttt{Cat-Proximity}: Proximity for categorical features as the total number of mismatches on categorical value between $\vecx^{cf}$ and $\vecx$ for each feature~\cite{DiverseCf}; 
    \item \texttt{Constraint Feasibility Score}: For Simple-BN and Sangiovese datasets, the harmonic mean of \% of CFs satisfying the two sub constraints (S1 and S2) of the given monotonic constraint, 
    $ \frac{2*S1*S2}{S1+S2} $, and for Adult we report \% of CFs satisfying C1 and C2 separately;
\item \texttt{Causal-Edge Score}: Log Likelihood of CFs w.r.t. a given causal edge distribution. 
    Causal-Edge-Score is defined only for $\tt Simple\-BN$ and $\tt Sangiovese$ where the true causal model is known. 
    \end{itemize}
We also evaluate on the Interpretability Score proposed in ~\cite{van2019interpretable}; details are  in Suppl. \textbf{C.1}.

\subsection{Results}

\textbf{Evaluating feasibility.} 
Figure~\ref{fig:main-metrics} shows \texttt{Constraint-Feasibility Score} for all datasets and  \texttt{Causal Edge Score} for Simple-BN and Sangiovese datasets,  averaged over 10 runs (other metrics are in the Suppl. \textbf{C.2}).  For Simple-BN dataset, \examplecf achieves the highest \texttt{Constraint  Feasibility score}, \modelcf achieve the highest Constraint Feasibility score on the Sangiovese dataset, while the \approxcf achieves the highest Constraint Feasibility score on the Adult dataset. 

All the methods achieve perfect score on the \texttt{Target-Class Validity} (refer to Suppl. \textbf{C.2}). However, across the three datasets, the methods designed to preserve feasibility (\modelcf, \approxcf, and \examplecf) perform better than \texttt{CEM} on the \texttt{Constraint -Feasibility Score}. That is, \texttt{CEM}
achieves a score of zero on simple-bn dataset and around 40\% on the Sangiovese and Adult dataset. 

The poor performance of \texttt{CEM} on Constraint  Feasibility across datasets suggests that feasibility cannot be solely captured by relying on the observed data likelihood. 
In the Adult dataset, the feasibility constraint requires Age to be increased, despite a correlation between low Age and High Income in the dataset, illustrating the fact that following observed distribution does not ensure feasibility.

We also compare on the Causal-Edge score metric that evaluates the log-likelihood of the generated CF wrt the true function in the causal graph. For Simple-BN dataset, we find that Model-Based method achieves the highest score, followed by Model-Approx and Example-Based. On Sangiovese, all methods are comparable. While the Model-Based and Model-Approx methods have explicit knowledge of the constraints, this result shows that the Example-Based method can also learn the constraint based on examples. CEM method has the lowest score. 
\\
\begin{wrapfigure}{r}{0.35\linewidth}
\centering
 \includegraphics[width=\linewidth]{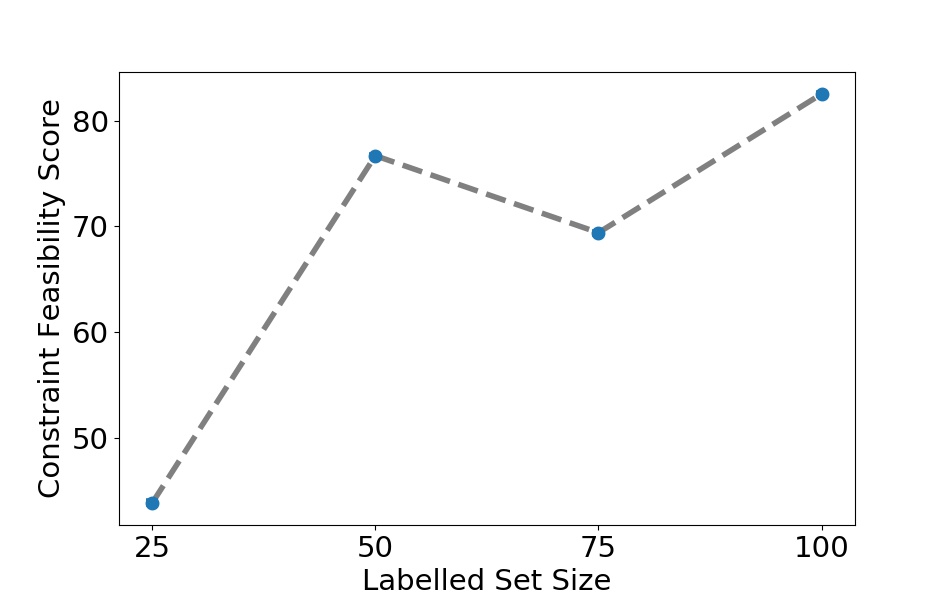}
\caption{Constraint-Feasibility score as the no. of labelled examples is increased for global constraints in \texttt{Adult}. 
}
\label{fig:other-metrics}
\end{wrapfigure}

\textbf{Example-Based CF: Constraint Feasibility increases with no. of labelled CFs.}
A key question for the Example-Based method is the number of labelled CF examples it needs. Using the Adult dataset and the non-decreasing Age constraint,  we show the \texttt{Constraint-Feasibility Score} of \examplecf as we increase the number of labelled CF examples (Figure~\ref{fig:other-metrics} (a)). For the global constraint, we find that the Feasibility Score increases with labelled inputs, reaching nearly $80\%$ with 100 labels. 
Compared to prior work that does a separate optimization for each CF~\cite{dhurandhar2018explanations,DiverseCf}, Example-Based method is also computationally efficient. We show its comparison to CEM in Suppl.D.
\\

\textbf{BaseVAE for CF generation}
Apart from feasibility, we conduct experiments with the BaseVAE method to test it for CF example generation on MNIST; the details can be found in the section Suppl.~\ref{sec:mnist-exp-base-vae}.

\section{Related Work}
Our work builds upon the literature on explainable ML~\cite{lime, shap} by focusing on a specific type of explanation through  counterfactual examples~\cite{wachter2017counterfactual}. Most CF generation methods rely on separate optimizations for each input, based on original features~\cite{russell2019efficient,DiverseCf},  a latent representation~\cite{joshi2019towards}, or using a generative adversarial network~\cite{liu2019generative}. We also build on the fundamental work on counterfactuals~~\cite{pearl2009book}.

To account for feasibility, different statistical notions of feasibility have been proposed, based on adherence to the training distribution~\cite{dhurandhar2018explanations}, distribution of the target class~\cite{van2019interpretable}, the likelihood of the intermediate points in reaching a CF example~\cite{poyiadzi2020face}, and on specifying which features can be changed~\cite{ustun2019actionable}. ~\cite{liu2019generative} rely on GAN's to restrict the explanations in semantically meaningful space. In a critical commentary, \cite{barocas2020hidden} raise concern that feasibility cannot be learned only from training data distribution. Related to our framework on expressing feasibility in terms of causal constraints, ~\cite{DiverseCf} point to the importance of causal constraints for feasibility, but do not provide a method for generating feasible CF examples. 
Our paper extends this line of work by formally defining feasibility, providing a theoretical justification of the counterfactual loss and proposing a VAE-based generative model that can preserve causal constraints. 
\cite{parafita2019explaining,karimi2020algorithmic,joshi2019towards} also focus on the role of causality towards feasible counterfactual explanations but their approach requires full knowledge of the causal model.

\section{Conclusion}
Feasibility in CF explanations is hard to  quantify.
In this work, we provided a generative model and two methods for modeling causal constraints. In future work, we will explore how to integrate domain knowledge and available data to learn causal constraints.







\small
\bibliographystyle{unsrt}
\bibliography{cfexplain}

\newpage    
\appendix
\appendix


\section{Supplementary Materials: Theorem 1 Proof}
\begin{theorem}
The evidence lower bound to optimize CF objective $\Pr(\vecx^{cf}| y', \vecx)$ for global feasibility is:
\begin{equation}
\begin{split}
\small
\ln \Pr(\vecx^{cf}| y', \vecx) \geq  \mathbb{E}_{Q(\vecz|\vecx, y')} \ln P( \vecx^{cf} | \vecz, y', \vecx) \\
- \mathbb{KL}(Q(\vecz|\vecx, y') ||P(\vecz| y', \vecx))
\end{split}
\end{equation}
\end{theorem}

\begin{proof}
An ideal counterfactual generation model approximates $\vecx$ (proximity) and generates $\vecx^{cf}$ that are valid w.r.t desired class $y'$. Thus for a model we seek to maximize $P(\vecx^{cf}| y', \vecx)$ where $P$ is the underlying probability distribution over $\mathcal{X}$.
\begin{equation}
\small
\begin{split}
   &  \ln{P(\vecx^{cf}| y', \vecx)}  \\
   & = \ln \int  P(\vecx^{cf}, \vecz| y', \vecx) d\vecz \\
                      & =    \ln  \int Q(\vecz|\vecx, y')\frac{ P(\vecx^{cf}, \vecz | y', \vecx) }{Q(\vecz|\vecx, y')} dz \\
                      & \geq \int Q(\vecz|\vecx, y')  \ln  \frac{P(\vecx^{cf}, \vecz | y', \vecx)}{Q(\vecz|\vecx, y')} dz \\
                      & = \mathbb{E}_{Q(\vecz|\vecx, y')} \ln \frac{P(\vecx^{cf}, \vecz |  y', \vecx)}{Q(\vecz|\vecx, y')}\\
                      & = \mathbb{E}_{Q(\vecz|\vecx, y')} \ln P(\vecx^{cf}| \vecz, y', \vecx) - \mathbb{E}_{Q(\vecz|\vecx, y')} \ln \frac{Q(\vecz|\vecx, y')}{P(\vecz| y', \vecx)}
\end{split}
\end{equation}
Where the inequality above is due to Jensen's inequality. 
Using the definition of KL-Divergence, 
\begin{equation*}
\small
\begin{split}
      \ln{P(\vecx^{cf}| y', \vecx)} &\geq \mathbb{E}_{Q(\vecz|\vecx, y')} \ln P(\vecx^{cf}| \vecz, y', \vecx) \\ 
      &- \mathbb{KL}({Q(\vecz|\vecx, y')}|| P(\vecz| y', \vecx))
\end{split}
\end{equation*}
\end{proof}

\section{Defining CF explanations as Interventions on a Structural Causal Model}

\begin{wrapfigure}{r}{0.25\linewidth}
    \includegraphics[width=\linewidth]{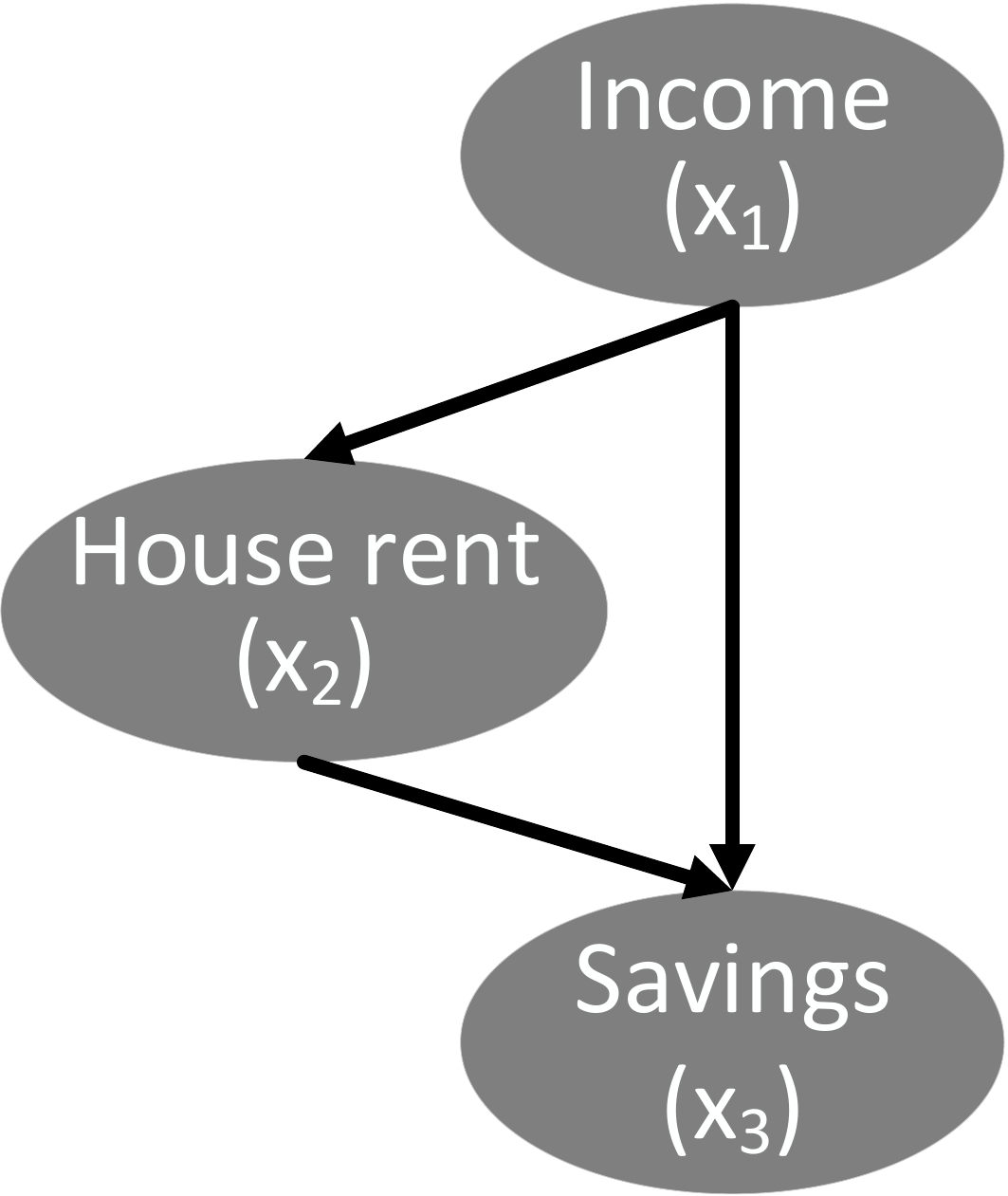}
    \caption{SCM describing the true causal relationships between three input features: Income, House rent, and Savings of a person. These input features are used by a pre-trained black-box ML model that we wish to explain.}
    \label{fig:example-scm}
\end{wrapfigure}
In Section~\ref{connections} we provided a novel distance metric that captures the feasibility of perturbations from the original input, based on a structural causal model (SCM) that describes causal relationships between different features. Since changing a feature $x_{v}$ can change the values of other features, the SCM helps us model the downstream changes due to change in $x_{v}$. However, unlike  the standard SCM intervention that cuts off all incoming edges (causes) on the perturbed features~\cite{karimi2020algorithmic}, we assume that the perturbed feature can also be affected by its other causes. This is because CF explanations are intended to present perturbations that are feasible in the real world and it is unlikely for a person to change a feature to any value independent of its causes. That is,  any suggested perturbation in a feature  that does not satisfy its relationship with its causal parents will not be possible in the real world. 

For example, consider the features represented by the SCM in Figure~\ref{fig:example-scm} where we assume that an ML model uses these features to predict a loan decision. Perturbing house rent can be considered as an intervention that affects savings of a person and may also lead to a counterfactual example by  changing their original decision outcome. However house rent cannot be changed independently of the person's income (its causal parent in Figure~\ref{fig:example-scm}): a simple feasibility constraint is that (perturbed)   house rent cannot exceed (optionally perturbed) income in any CF example. Therefore, rather than independent interventions that are suitable when estimating causal  effect (e.g., estimating the effect of a drug treatment through a randomized experiment), we employ a modified version where the causes continue to affect the perturbed feature. 

In situations where an independent intervention on features is possible, those features can be considered as exogenous variables  $U$ when computing the $\operatorname{DistCausal}$ metric from Section~\ref{Causal-CFVAE} (under a modified SCM wherein incoming edges to such features are cut off).


\textbf{Comparison to the standard proximity metric for CF explanations.}
Note that unlike the standard distance metric for proximity, we do not compare the proposed $x_v^{cf}$ to the original input's feature value $x_v$, but rather compare $x_v^{cf}$ to its predicted value $f(x_{v_{p1}}^{cf},..,x_{v_{pk}}^{cf})$ based on the values of its causal parents in $\vecx^{cf}$. Here $f: \mathbb{E}[x_v|x_{vp1},..., x_{vpk}]$ is the conditional expected value based on the SCM.

When the distance metric  is $\ell_2$, the above can be equivalently written as the distance between the \emph{relative} change between $x_v$ and $x_v^{cf}$, and the expected change between  $f(x_{v_{p1}},..,x_{v_{pk}})$ and  $f(x_{v_{p1}}^{cf},..,x_{v_{pk}}^{cf})$ as predicted by the SCM.
\begin{equation}
    \begin{split}
        \Delta_v &=  x_v^{cf} - x_v \\
        \Delta_{predictedv} &= ( f(x_{v_{p1}}^{cf},..,x_{v_{pk}}^{cf}) +\epsilon_1) -  (f(x_{v_{p1}},..,x_{v_{pk}}) + \epsilon_2)\\
        \mathbb{E}[\Delta_{predictedv}] &=f(x_{v_{p1}}^{cf},..,x_{v_{pk}}^{cf})  -  f(x_{v_{p1}},..,x_{v_{pk}}) 
    \end{split}
\end{equation}
where $\epsilon_1$ and $\epsilon_2$ are mutually independent zero-mean errors.  Then the distance can be written as,
\begin{equation}
    \begin{split}
        \operatorname{DistCausal_v}(x_{v}, x_{v}^{cf}) &= \operatorname{Dist}(\mathbb{E}[\Delta_v], \mathbb{E}[\Delta_{predictedv}]) \\
                                                       &= \ell_2(x_v^{cf}, f(x_{v_{p1}}^{cf},..,x_{v_{pk}}^{cf}))
    \end{split}
    \end{equation}
    averaged over different input values. 

\section{Implementation Details and Results on Bayesian Networks and Adult Dataset}
Here we provide implementation details and additional results on the Bayesian network and Adult datasets.
\subsection{Implementation Details}
\subsubsection{ML Model Architecture}
We first describe the architecture of the ML model. The ML model for all the datasets was trained for 100 epochs, batch size 32, learning rate $10^{-3}$ with Adam optimizer and Cross Entropy Loss, with  80-10-10\% training, validation and test datasets. It achieved test accuracy of 87\% on Simple-BN, 83.5\% on Sangiovese, and 89.3\% on Adult dataset.

Each ML model comprises of two layers as shown below:
\begin{itemize}
    \item Hidden Layer(data-size, hidden-dim)
    \item Hidden Layer(hidden-dim, num-classes)
\end{itemize}
 The values of \textit{hidden-dim} is chosen as 10 and value of \textit{num-classes} is 2 as it is  a binary classification task across the three datasets.

\subsubsection{BaseVAE  Architecture}
\label{sec:BaseGenCF-Implementation}
Here we provide the implementation details of the base variational encoder decoder used in all our different methods. Both the encoder and decoder are modeled as Neural Networks (NN) with multiple hidden layers and non linear activations. Encoder comprises of two Neural Networks: one NN is used to estimate the mean and other NN is used to estimate the variance of posterior distribution $q(z|x,y_k)$. Both the networks for estimating mean and variance have the same architecture as described below, with the only difference that the variance network having an additional Sigmoid activation at the end to ensure the variance is positive. Similarly, decoder comprises of a neural network to estimate the counterfactual from the latent encoding and the target class.

\textbf{Encoder Architecture}:

\begin{itemize}
    \item Hidden Layer 1(DataSize+1, 20), BatchNorm, Dropout(0.1), ReLU
    \item Hidden Layer 2(20, 16), BatchNorm, Dropout(0.1), ReLU
    \item Hidden Layer 3(16, 14), BatchNorm, Dropout(0.1), ReLU
    \item Hidden Layer 4(14, 12), BatchNorm, Dropout(0.1), ReLU
    \item Hidden Layer 5(12, LatentDim)
    \item Sigmoid (In case of variance network only)
\end{itemize}

\textbf{Decoder Architecture}:

\begin{itemize}
    \item Hidden Layer 1(LatentDim+1, 12), BatchNorm, Dropout(0.1), ReLU
    \item Hidden Layer 2(12, 14), BatchNorm, Dropout(0.1), ReLU
    \item Hidden Layer 3(14, 16), BatchNorm, Dropout(0.1), ReLU
    \item Hidden Layer 4(16, 20), BatchNorm, Dropout(0.1), ReLU
    \item Hidden Layer 5(20, DataSize), Sigmoid
\end{itemize}

The latent space dimension ($LatentDim$) is set to 10 for all the different methods and datasets. Both the encoder and the decoder are conditioned on the target counterfactual class. Hence, the \texttt{Hidden Layer 1} in the Encoder takes the data point concatenated with the target class as the input. Similarly, the \texttt{Hidden Layer 1} in the Decoder takes the latent sample concatenated with the target class as the input. 

For all datasets, we use 
the SGD optimizer with learning rate $10^{-3}$ for 50 epochs. The batch size for the different datasets are: Simple-BN : 64, Sangiovese : 512, Adult : 2048. 

\subsubsection{Contrastive Explanantions}
For experiments involving the Contrastive Explanations (\texttt{CEM} ) method ~\cite{dhurandhar2018explanations}, we used the implementation provided by the open source library ALIBI~\cite{alibi}. Since the choice of auto encoder is not specified in ALIBI, we use an auto encoder with the same architecture as defined above in \textit{Base VAE Architecture} section for a fair comparison. The only difference being that the Encoder and Decoder are not conditioned on the target class. 

\subsubsection{HyperParameter tuning} 

For a fair comparison between methods, we optimize hyperparameters using the validation set and use random search for 100 iterations. Since an ideal counterfactual needs to satisfy feasibility, target-class validity, and proximity, we select the hyperparameters that lead to maximum feasibility, while still obtaining more than 90\% target class validity and proximity at least $\tau$. $\tau$ was conservatively selected to remove models that result in much lower proximity than the BaseVAE method (and thus will be less useful in practice).
In our experiments, we found that all methods achieved near 100\% target-class validity.

The following threshold values of $\tau$ were used depending on the dataset:
\begin{itemize}
    \item BN1: 9.0 (Cont. Proximity)
    \item Sangiovese: 20.0 (Cont. Proximity)
    \item Adult: 8.0 (Cont. Proximity) and 5.0 (Cat. Proximity)
\end{itemize}

We did not include the other metrics like Interpretability Score and Causal Edge score during the hyperparameter tuning, to allow an independent evaluation on those metrics. 

The final optimal values of the hyper parameters for each dataset are reported in Table \ref{tab:hyperparam}. As per the \texttt{BaseGenCFLoss} equation (section 3.1), we have the two hyperparameters $\beta$ and $\lambda$  common across all the approaches; which we refer to as \texttt{Margin} and \texttt{Validity} in the table \ref{tab:hyperparam}. 

Additionally, there is an extra  hyperparameter involved with approaches like \modelcf, \approxcf, and \examplecf. We denote this extra hyperparameter as \texttt{Feasibility} in the Table \ref{tab:hyperparam}. For \examplecf, it corresponds to $\lambda_{o}$ which controls the trade-off between modeling CF and the oracle (Section 3.2). 
In the case of \modelcf, it corresponds to the trade-off between Exogenous and Endogenous Loss terms in \texttt{DistCausal} loss term, as described below:
\begin{equation*} 
\label{CausalProximity}
\tiny
\begin{split}
\operatorname{DistCausal}(\vecx, \vecx^{cf}) = \sum_{u \in U}^{}  \operatorname{Dist_u} (\vecx_{u}^{cf}, \vecx_{u}) + \lambda_{s}*\sum_{v \in V}^{} \operatorname{DistCausal_v}(\vecx_{v}, \vecx_{v}^{cf}) 
\end{split}
\end{equation*}
In the case of \approxcf, for Unary constraints, it corresponds to the trade-off between \texttt{BaseGenCFLoss} and the Hinge Loss on feature of interest. For the Binary constraints, it would follow the same procedure as for \modelcf. 

For the Contrastive Explanations (CEM) method, the optimal hyperparameters for each dataset are as follows: (refer to  open source library ALIBI~\cite{alibi} for details regarding each hyper parameter)

\begin{itemize}
    \item Simple BN: Beta (0.608), Kappa (0.021), Gamma (8.0), CSteps (3), Max Iterations (1000)
    \item Sangiovese: Beta (0.652), Kappa (0.041), Gamma (9.0), CSteps (5), Max Iterations (1000)
    \item Adult: Beta (0.911), Kappa (0.241), Gamma (0.0), CSteps (9), Max Iterations (1000)
\end{itemize}

\begin{table}[t]
\small
\caption{Hyperparameter tuning decription for all the methods and datasets}
\centering
\begin{tabular}{llccc}
\toprule
\textbf{Dataset} & \textbf{Method}    &  Margin & Validity  & Feasibility  \\ \midrule
& \approxcf  & 0.087    & 96  & 0.1 \\ 
Simple-BN  & \examplecf  & 0.15  & 150   & 2350 \\  
 & \modelcf  & 0.015    & 85  & 55 \\  \midrule

 & \approxcf  & 0.306   & 71  & 73 \\ 
Sangiovese & \examplecf  & 0.02  & 25   & 1085 \\  
 & \modelcf & 0.319    & 89 & 77     \\  \midrule

Adult-Age & \approxcf  & 0.764   & 29  & 192 \\ 
 & \examplecf & 0.084  & 159   & 5999 \\ \hline
    
Adult-Age-Ed & \approxcf  & 0.344   & 76 & 87 \\ 
 & \examplecf  & 0.117  & 175   & 3807 \\ \bottomrule
\end{tabular}
\label{tab:hyperparam}
\end{table}

\subsubsection{Evaluation Metrics}
We define the following metrics to evaluate CF explanations; considering the case of N data points $\{x_{i}\}$, with K CF's $\{x^{cf}_{i,j}\}$ sampled for each data point.

\begin{itemize}[itemsep=0pt,leftmargin=*,topsep=2pt]
    
    \item  \texttt{Target-Class Validity}: \% of CFs whose predicted class by the ML classifier is the same as the target class: $ \frac{ \sum_{i=1}^{N} \sum_{j=1}^{K} \mathbbm{1}[f(x_{i,j}^{cf}) == t_c] }{N*K}$.
    
    \item  \texttt{Cont-Proximity}: Proximity for continuous features as the average $\ell_1$-distance between $\vecx^{cf}$ and $\vecx$ in units of median absolute deviation for each features~\cite{DiverseCf}; It is multiplied by $(-1)$ so that higher values are better. $\frac{ \sum_{i=1}^{N} \sum_{j=1}^{K} \sum_{p=1}^{d_{cont}}  \frac{ x_{i,j}^{cf, p} - x_{i}^{p} }{MAD_{p}} }{N*K*d_{cont}}$
    
    \item     \texttt{Cat-Proximity}: Proximity for categorical features as the total number of mismatches on categorical value between $\vecx^{cf}$ and $\vecx$ for each feature~\cite{DiverseCf}; It is multiplied by $(-1)$ so that higher values are better. $\frac{ \sum_{i=1}^{N} \sum_{j=1}^{K} \sum_{p=1}^{d_{cat}}   \mathbbm{1}[x_{i,j}^{cf,p} \neq x_{i}^{p}] }{N*K*d_{cat}}$

    \item   \texttt{Constraint Feasibility Score}: For Simple-BN and Sangiovese datasets, the constraints mentioned can be observed as a combination of two sub constraints: X1 and X2. For example, in the case of Simple-BN dataset, X1 corresponds to ``{\small ($\vecx_1,\vecx_2$ increase $\implies$ $\vecx_3$ increases)}''; while X2 correspond to ``{\small ($\vecx_1,\vecx_2$ decrease $\implies$ $\vecx_3$ decrease)}''
    
    Hence, to ensure good performance at satisfying both sub-constraints, we define the following metric for constraint feasibility: $ \frac{2*S1*S2}{S1+S2} $
    
    where S1, S2 represent the \% of CFs satisfying the sub constraints X1, X2 respectively.
     
    However in the case of Adult dataset, due to the additional non monotonic constraint of Education level cannot decrease, we simply report the percentage of Counterfactual satisfying the complete constraint $\tt C2$ on the Adult dataset.
    
    For unary constraints, like $\tt C1$ on the Adult dataset, we always report the percentage of Counterfactuals satisfying the constraint. 
    
    \item \texttt{Causal-Edge Score}: Ratio of the Log Likelihood of CFs $x^{cf}$ and the Log Likelihood of the original data point $x$  w.r.t. to the set of given causal edges distribution $V$ : $ \frac{ \sum_{i=1}^{N} \sum_{j=1}^{K} \sum_{v \in V}^{}  \frac{ \log p(x^{cf}_{v} | x^{cf}_{vp1},..,x^{cf}_{vpk})}{ \log p(x_{v} | x_{vp1},..,x_{vpk})  } }{N*K} $

    \item \texttt{Interpretability Score: } The IM1 metric as defined by \cite{van2019interpretable}, the ratio of the reconstruction loss of CFs given the target class Auto Encoder and the reconstruction loss given the original class Auto Encoder :  $ \frac{ \sum_{i=1}^{N} \sum_{j=1}^{K} \frac{\|\vecx^{cf} - AE_{t}(\vecx^{cf})\|}{\|\vecx^{cf} - AE_{o}(\vecx^{cf})\|} }{N*K} $. Thus, lower  IM1 score is  better. 
\end{itemize}

\subsubsection{ Constraint Modelling in \approxcf}
For the case of Simple-BN and Sangiovese dataset, the feasibility constraint is monotonic, hence we use the \textit{Binary Constraints} formulation of \approxcf, as described in the Section 2.4 in the main submission. 

For the case of Adult Dataset, the constraint C1 is modelled using the \textit{Unary Constraints} formulation, with a Hinge Loss on the feature Age. 
The constraint C2 in Adult Dataset is more complex than the previous constraints, since the feature Education is categorical. We model the constraint C2 under the \textit{Unary Constraints} formulation, since it can be viewed as combinations of two unary constraints: Age cannot decrease and Education cannot decrease. The Hinge Loss on categorical variable Education is implemented by converting the embedding of categorical variable Education into a continuous value.
We rank different education levels with increasing score and take a weighted sum of the categorical embedding with the scores assigned for each education category.  Hence, we get a continuous score for education feature embedding which is representative of the level/rank of education. Now, we can apply the same Hinge Loss on the continuous values of education feature to put penalty on counterfactuals that decrease the level of education.

The vector we used for ranking different educations levels is as follows:
\begin{itemize}
    \item HS-Grad, School: 0
    \item Bachelors, Assoc, Some-college: 1
    \item Masters:2 
    \item Prof-school, Doctorate: 3
\end{itemize}

\begin{figure*}[t]
\begin{tabular}{ccc}
    \includegraphics[width=0.30\textwidth]{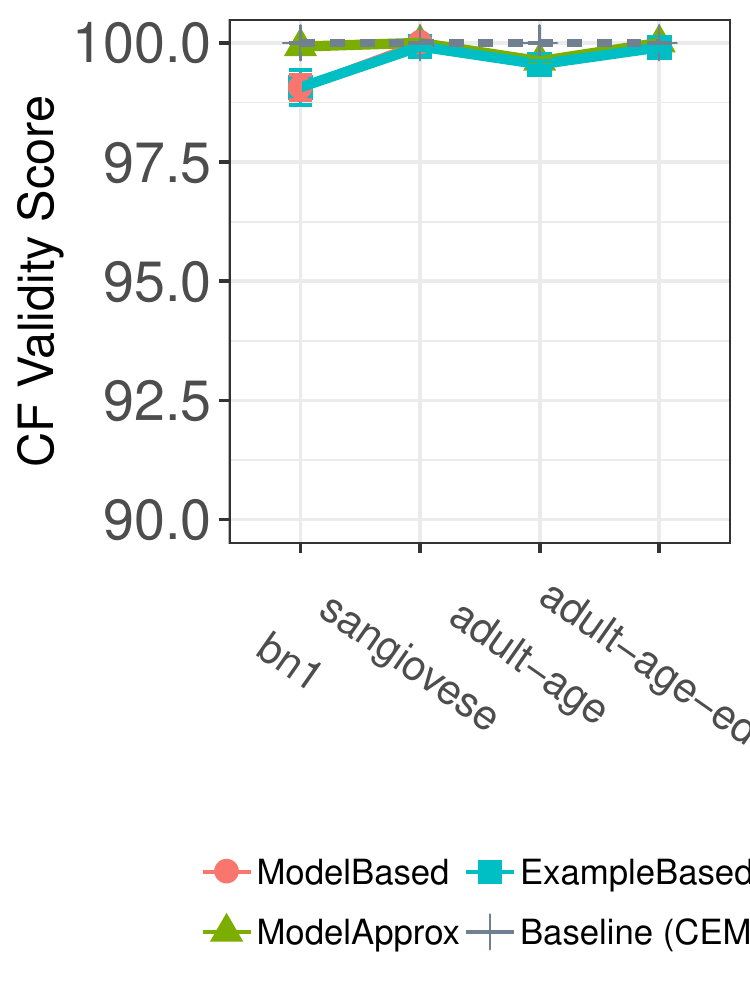} &
    \includegraphics[width=0.30\textwidth]{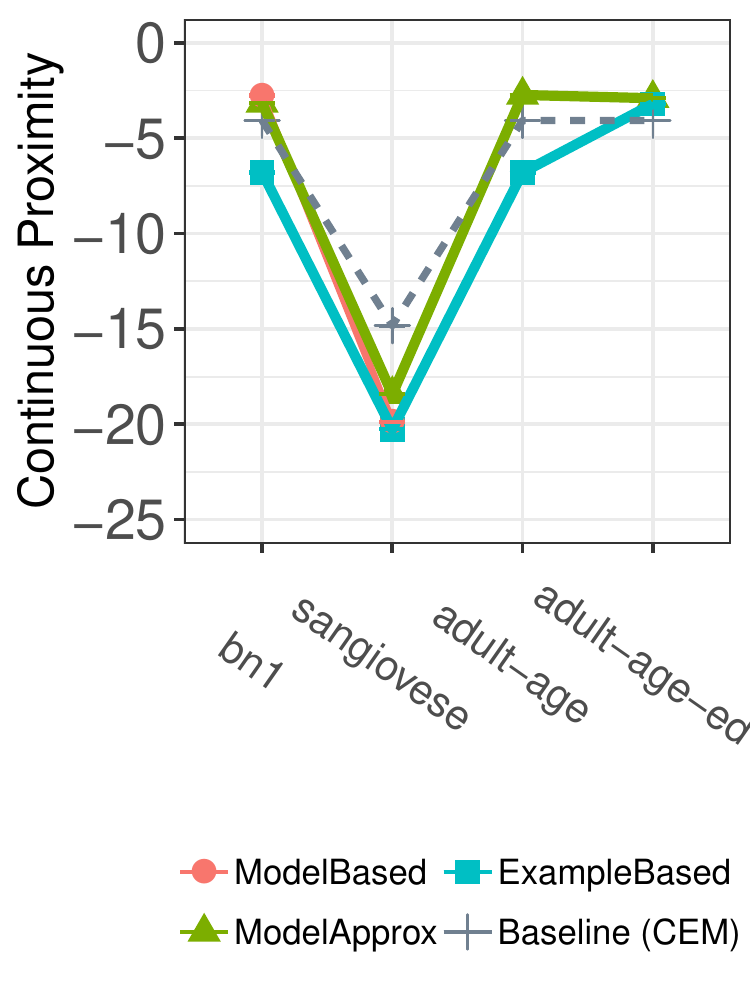} &
\includegraphics[width=0.30\textwidth]{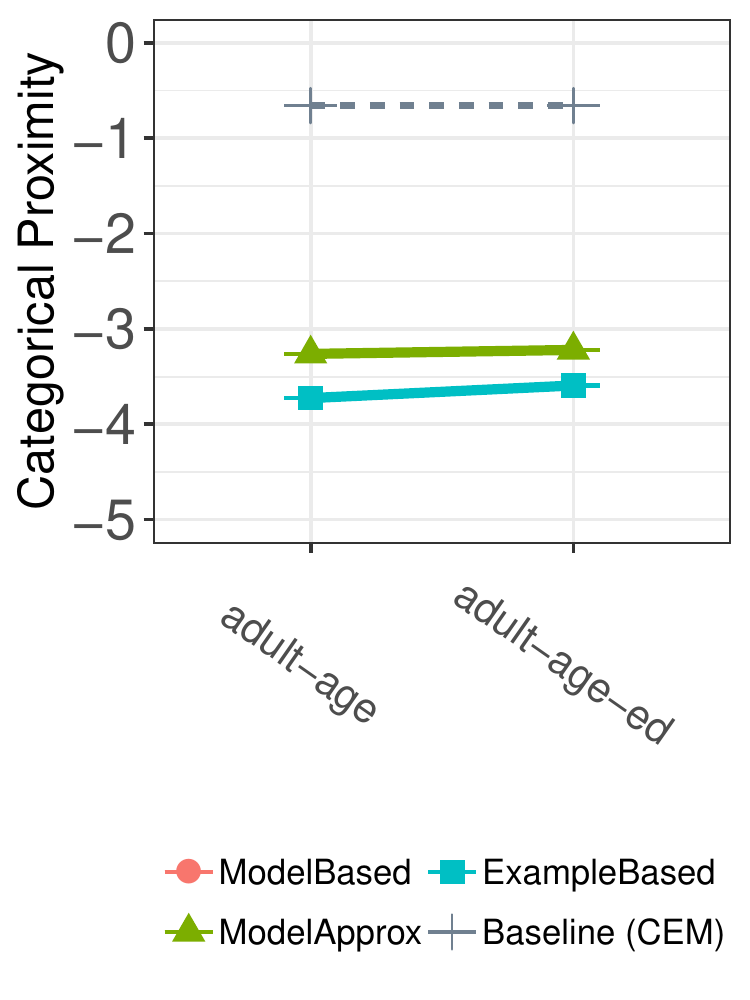} 
\\
(a) & (b) &(c)
\end{tabular}
\caption{Validity,  Continuous Proximity and Categorical Proximity metrics for different CF explanation methods.}
\label{fig:extra-metrics-2}
\end{figure*}

\subsection{Evaluation Results on Additional Metrics}

\textbf{Target-Class Validity: }
Figure \ref{fig:extra-metrics-2}(a) shows the performance of the methods on \texttt{Target-Class Validity}. All the methods achieve near perfect score on this metric across datasets.





\textbf{Continuous Proximity: }
Figure~\ref{fig:extra-metrics-2}(b) shows the methods evaluated on the Cont-Proximity metric. The dataset \texttt{Sangiovese} shows an interesting trend where the approaches with higher Constraint-Feasibility score (\modelcf, \approxcf, \examplecf) perform worse on Cont-Proximity metric as compared to the approaches with lower Constraint-Feasibility score (Baseline (CEM)). This suggests increasing feasibility might induce a trade-off with the continuous proximity in some cases.
 

\textbf{Categorical Proximity: }
Figure \ref{fig:extra-metrics-2}(c) shows the methods evaluated  on the Cat-Proximity metric. The dataset \texttt{Simple-BN} and \texttt{Sangiovese} do not contain any categorical variables, hence we do not include them for this analysis. \texttt{CEM} performs better than other methods on Categorical proximity in both the cases of age (C1) and age-ed (C2) constraint. This may be because CEM tends to not vary categorical features.  Table \ref{tab:qual-examples} provides an example  via generated CF examples for the Adult dataset: \texttt{CEM} does  not change  categorical features like Occupation, MaritalStat, Race unlike \approxcf and \examplecf. Along with Figure~\ref{fig:extra-metrics-2}(b), this result demonstrates the trade-off between feasibility and proximity. \texttt{CEM} is worse at preserving constraints (e.g., in Table~\ref{tab:qual-examples}, CF examples by CEM do not increase the value of Age while increasing the Education level to Masters), but achieves higher proximity score for categorical variables. 

\begin{table*}[t]
\tiny
\caption{Examples of generated counterfactuals on the modified Adult dataset. \examplecf and \approxcf were trained to preserve the Education-Age causal constraint}\label{tab:qual-examples}

\begin{tabular}{l|l|llllllll}
\toprule
\textbf{Adult}          & Method    & AgeYrs & Education   & Occupation   & WorkClass     & Race  & HrsWk & MaritalStat & Sex    \\ \hline

\begin{tabular}[c]{@{}l@{}}Original input\\ (outcome: \textless{}=50K)\end{tabular} & & 41  & Some-college     & Blue-collar      & Private       & Other & 40   & Single   & Male \\ \midrule

\textbf{}                                                                           &  & 46  & Masters & White-Collar        & Private & White & 41   & Married   & Male \\
Counterfactuals                                                                     & \textbf{\approxcf} & 45  & Some-college & White-Collar        & Private & White & 40   & Married   & Male \\
(outcome: \textgreater{}50K)                                                        &  & 47  & Masters & White-Collar        & Private & White & 41   & Married   & Male \\

\midrule

\textbf{}                                                                           &  & 45  & Prof-school & White-Collar        & Private & White & 42   & Married   & Male \\
Counterfactuals                                                                     & \textbf{\examplecf} & 44  & Masters & White-Collar        & Private & White & 42   & Married   & Male \\                                                         
(outcome: \textgreater{}50K)                                                        &  & 47  & Prof-school & White-Collar        & Private & White & 43   & Married   & Male \\

\midrule

\textbf{}                                                                           &  & 41  & Masters & Blue-Collar & Private & Other & 39   & Single  & Male \\
Counterfactuals                                                                     & \texttt{ Baseline (CEM)} & 41  & Some-college & Blue-Collar & Other & Other & 39   & Single  & Male \\                 
(outcome: \textgreater{}50K)                                                        &  & 40  & Masters & Blue-Collar & Private & Other & 41   & Single  & Male \\
\bottomrule
\end{tabular}
\end{table*}

\begin{figure}[h]
    \centering
    \includegraphics[width=0.3\linewidth]{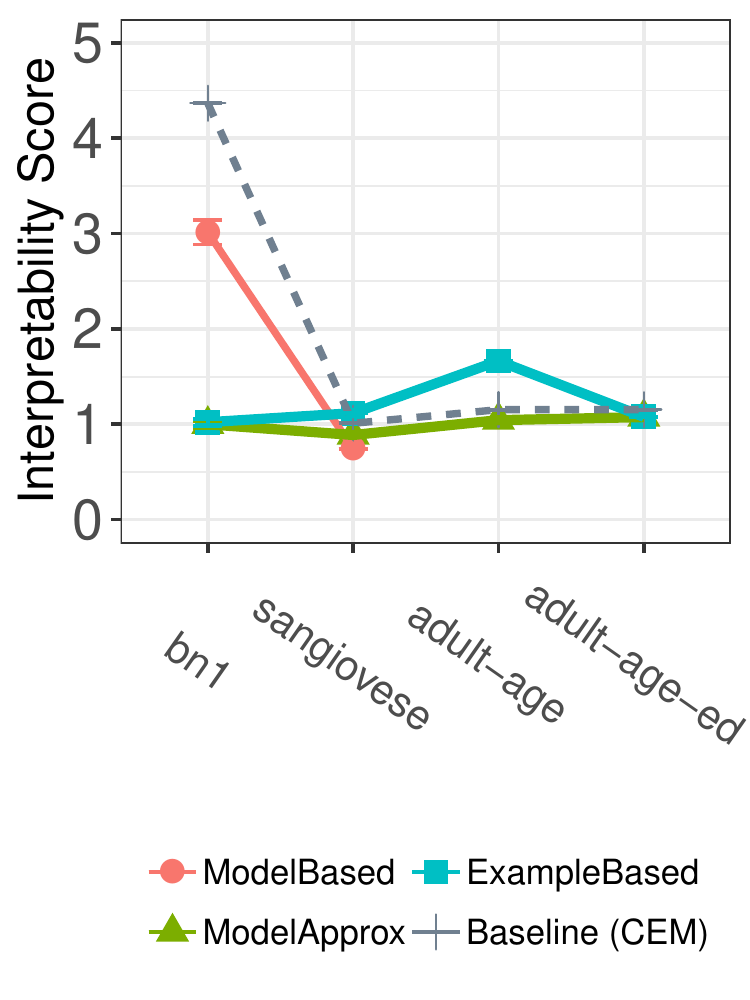}
    \caption{Interpretability score (IM1) for different CF generation methods. Lower IM1 score is better. }
    \label{fig:imscore}
\end{figure}
\textbf{Interpretability Score:}
Figure~\ref{fig:imscore} shows the performance of the methods on the Interpretability score (IM1 metric). \modelcf and \approxcf consistently perform better than Baseline (CEM) across different datsaets, which suggests that \modelcf and \approxcf  do not generate counterfactuals that are far away from the data distribution while preserving feasibility constraints. Also, \examplecf performs worse than Baseline (CEM) on the \texttt{adult-age} dataset despite obtaining a higher Constraint-Feasibility score (Figure~\ref{fig:main-metrics}(a)). This suggests that generating counterfactuals closer to the data distribution is not guaranteed to provide feasibility.


\begin{figure*}[t]
\begin{tabular}{ccc}
\includegraphics[width=0.3\textwidth]{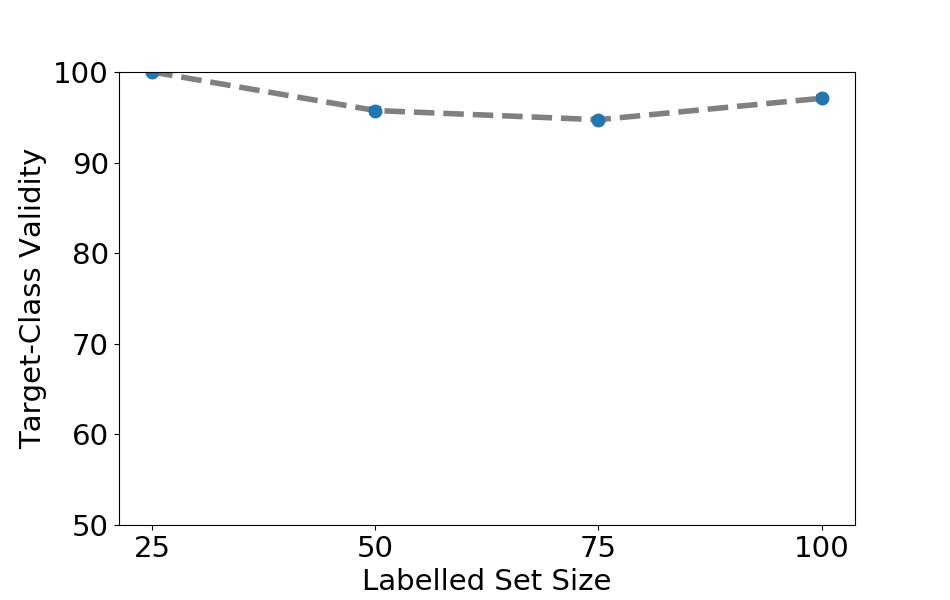} &
\includegraphics[width=0.3\textwidth]{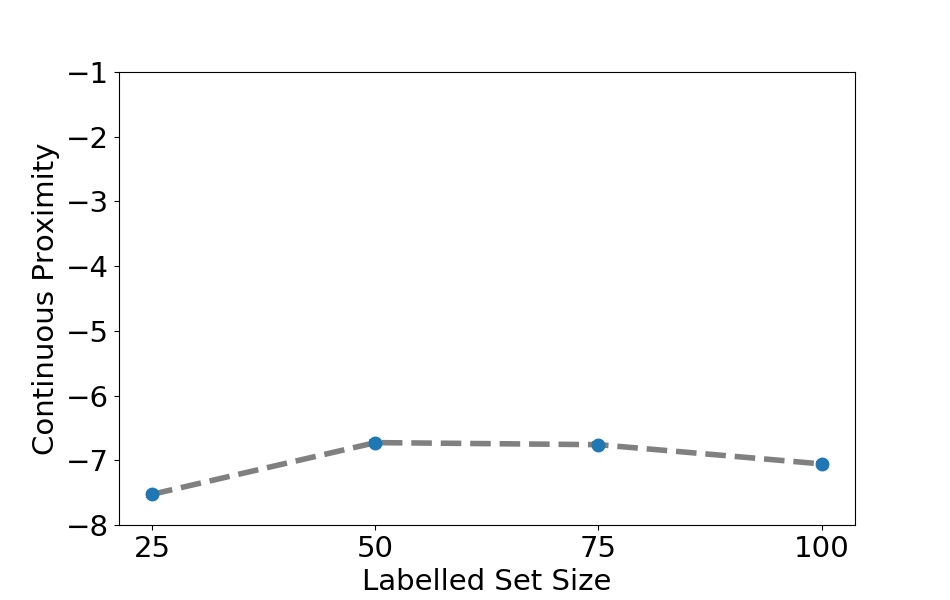} &
\includegraphics[width=0.3\textwidth]{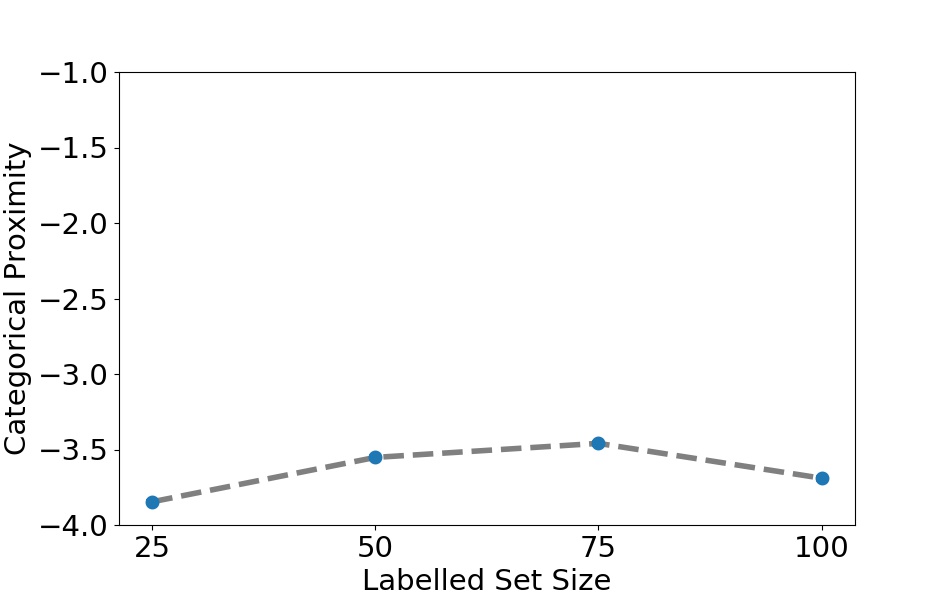} 
\\
(a) & (b) &(c)
\end{tabular}
\caption{Target-Class Validity, Continuous Proximity and Categorical Proximity as \examplecf is trained on more labelled examples  in the Adult dataset.}
\label{fig:extra-adult-global}
\end{figure*}

\section{Computational Efficiency of \examplecf}

Besides feasibility, \examplecf is computationally faster than past methods like \texttt{CEM}, since it uses a generative model. In Figure~\ref{fig:oracle-time}, we show the time required to generate counterfactual examples for $k$ inputs for the Adult dataset and find that \examplecf takes less time per CF example as $k$ increases. \texttt{CEM} has a non-trivial execution time for each input while \examplecf  takes time for initial training but then negligible time for every new input. 
Since \modelcf and \approxcf also rely on a VAE architecture, they are similarly efficient to \examplecf.

\textbf{Additional metrics for Figure 3(a).}
In Figure~\ref{fig:other-metrics}, we showed the Constraint-Feasibility score as the number of labelled examples from the \texttt{Adult} dataset are increased for the \examplecf method. Here we show additional metrics in    
Figure (\ref{fig:extra-adult-global}): Target-Class Validity, Continuous Proximity and Categorical Proximity. 
While we saw a substantial increase in the \texttt{Constraint-Feasibility} Metric (Figure 3 in the main submission), we find that other metrics on validity and proximity do not change much,  as the number of labelled examples increases from 25 to 100. 

\begin{figure*}[tb]
\centering
\includegraphics[width=0.33\textwidth]{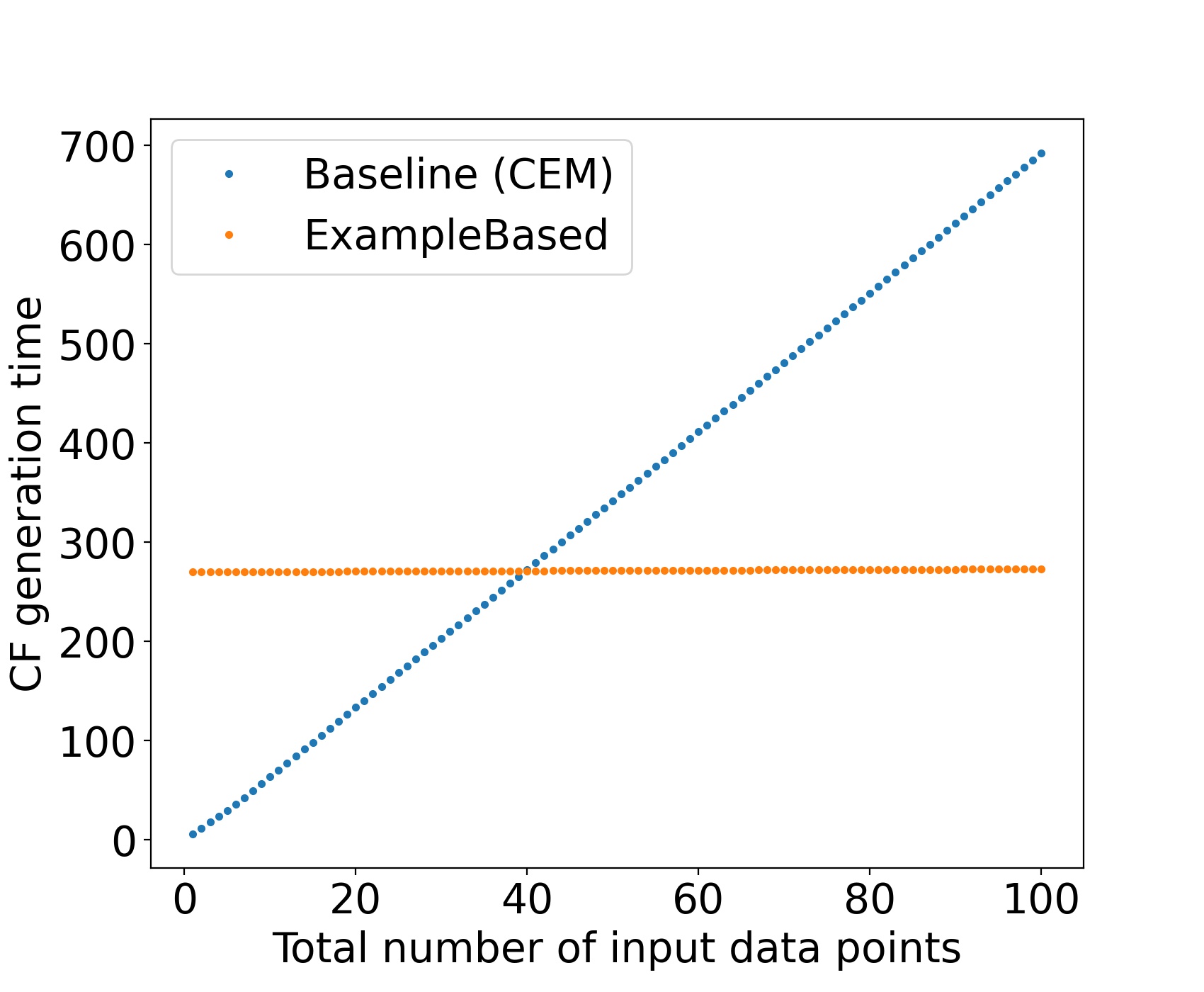}
\caption{Comparison of the time taken to generate CF examples for the proposed \examplecf method and the baseline CEM method. After about 40 inputs, \examplecf is faster than CEM for generating CF examples.} 
\label{fig:oracle-time}
\end{figure*}

\section{Supplementary Materials: Applying Base VAE to an Image Dataset}
\label{sec:mnist-exp-base-vae}

Finally, to show the generality of the proposed method, we apply the BaseVAE CF generator on the   MNIST  dataset~\cite{lecun1998gradient} which contains 70,000 labeled 28x28 images of handwritten digits between 0 and 9. As the ML model to be explained, we train a neural network model on the dataset to predict the digit classes. For explaining this model, we consider the counterfactual generation task on 5 digits (2, 3, 4, 5, 8); with the respective target counterfactual classes (3, 5, 9, 3, 9) as shown in  Figure~\ref{fig:mnist-examples}. We take a subset of 100 samples for each of the 5 digits for training the \texttt{BaseVAE} method. 
Details on the model architecture for the ML model and \texttt{BaseVAE} are provided in section~\ref{sec:mnit-implement}
 
\begin{figure}[tb]
\centering
\begin{tabular}{ccccc}
\includegraphics[width=0.05\textwidth]{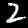}
&
\includegraphics[width=0.05\textwidth]{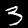}
&
\includegraphics[width=0.05\textwidth]{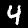}
&
\includegraphics[width=0.05\textwidth]{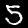}
&
\includegraphics[width=0.05\textwidth]{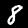}
\\
\includegraphics[width=0.05\textwidth]{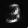}
&
\includegraphics[width=0.05\textwidth]{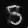}
&
\includegraphics[width=0.05\textwidth]{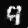}
&
\includegraphics[width=0.05\textwidth]{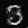}
&
\includegraphics[width=0.05\textwidth]{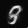}
\\
\end{tabular}
\caption{CF examples generated from MNIST images; the top row denotes the image and the bottom row  denotes the associated counterfactual. In each case, a target class (\emph{digit}) for the counterfactual was provided to the BaseVAE method.}
\label{fig:mnist-examples}
\end{figure}

\subsection{Results}

Figure~\ref{fig:mnist-examples} shows the counterfactuals generated by our \texttt{BaseVAE} approach. Additionally, for evaluation we use two metrics defined by \cite{van2019interpretable}:\texttt{IM1} metric which measures the ratio between the reconstruction of $x_{cf}$ using AutoEncoders trained on target class and original class data: 
{
\small
   $ IM1= \frac{\|\vecx^{cf} - AE_{t}(\vecx^{cf})\|}{\|\vecx^{cf} - AE_{o}(\vecx^{cf})\|} $
    }
 , and the \texttt{IM2} metric, that measures the difference between the reconstruction of $x_{cf}$ using AutoEncoders trained on the target class and all classes: 
 {
\small
   $ IM2= \frac{\|AE(\vecx^{cf}) - AE_{t}(\vecx^{cf})\|}{\|\vecx^{cf}\|} $
    }.
    
To ensure a fair comparison; we used the same ML model and AutoEncoder architecture as \cite{van2019interpretable} for computing the Interpretability Metrics (Table \ref{tab:mnist-res}). The results reported for \texttt{BaseVAE} approach are the mean and standard deviation over 10 runs.

\texttt{BaseVAE} achieves a better IM1 and IM2 score than other approaches, even though it does not explicitly have an autoencoder term in its loss, unlike (B, C, D, E, F) approaches that use the AE or Prototype loss. These results are probably since our \texttt{BaseVAE} generative model,  being trained on the full train data, is able to capture the feature distribution better. 

That said, it is harder to interpret IM2 metric; Table \ref{tab:mnist-add-res} breaks down the performance of the BaseGenCF method on the \texttt{IM1} and \texttt{IM2} metric by anaylzing the performance on its different components. IM1 Numerator corresponds to the reconstruction loss with the target class Auto Encoder: $ \|\vecx^{cf} - AE_{t}(\vecx^{cf})\| $; while the IM1 Denominator corresponds to the reconstruction loss with the original class Auto Encoder: $\|\vecx^{cf} - AE_{o}(\vecx^{cf})\|$. 

IM2 Numerator corresponds to the difference between the reconstruction loss using AutoEncoders trained on the target class and all classes: $\|AE(\vecx^{cf}) - AE_{t}(\vecx^{cf})\|$; while the the IM2 Denominator corresponds to the norm of the counterfactual: $\|\vecx^{cf}\|$. Note that any method's performance on the IM2 metric would depend a lot on the norm of counterfactuals generated (IM2 Denominator), which may not be desirable. This limits a proper interpretation of the IM2 Metric as proposed by \cite{van2019interpretable}. 

To enable a better comparison, here we report the performance of our method BaseVAE on both components of IM2 metric,  We do not have  results of \cite{van2019interpretable} on the different components of IM1 and IM2 metric, hence it is unclear  why we obtain substantially better results  on IM2 Metic ( Table~\ref{tab:mnist-res} ). It may be due to less reconstruction error difference (IM2 Numerator) or due to higher norm counterfactuals (IM2 Denominator) generated by our method. 

\textbf{Time Complexity: }
In addition, the time taken to generate CFs by BaseVAE is order of magnitude lower than other approaches.   That said, we do have a fixed training time ( \textbf{172.81} $\pm$ \textbf{3.97 }seconds ); thus our approach will be efficient for deployments where it can be trained once and used for generating multiple CFs for different inputs.

\begin{table}[h!]
\small
 \caption{ Results for the MNIST dataset. Metrics for other approaches are from Table 1 in ~\cite{van2019interpretable} }
 \centering
 \begin{tabular}{|l| c |c | c | c | c|} 
 \hline
\textbf{Method}  & {\textbf{Time (s)}} & {\textbf{Gradient Steps }} &  {\textbf{IM1}} & {\textbf{IM2(*10)}} \\ 
  \hline   
  A: V & $13.06 \pm 0.23$ & $5158 \pm 82$ & $1.56 \pm 0.03$ & $1.65 \pm 0.04$ \\
 B: VA & $8.40 \pm 0.38$ & $2380 \pm 113$ & $1.36 \pm 0.02$ & $1.60 \pm 0.03$ \\
 C: VP & $2.37 \pm 0.09$ & $751 \pm 31$ & $1.23 \pm 0.02$ & $1.46 \pm 0.03$ \\
 D: VAP & $2.05 \pm 0.08$ & $498 \pm 27$ & $1.26 \pm 0.02$ & $1.29 \pm 0.03$ \\
 E: P & $4.39 \pm 0.04$ & $1794 \pm 12$ & $1.20 \pm 0.02$ & $1.52 \pm 0.03$ \\
 F: PA & $2.86 \pm 0.06$ & $773 \pm 16$ & $1.22 \pm 0.02$ & $1.29 \pm 0.03$ \\
 \hline
BaseVAE & $\textbf{0.033} \pm \textbf{0.001}$ & $\textbf{600 (Training)}$ &  $\textbf{1.07} \pm \textbf{0.07}$ & $\textbf{0.12} \pm \textbf{0.03}$ 
\\
 \hline
 \end{tabular}
 \label{tab:mnist-res}
 \vspace{-0.5em}
\end{table}

\

\begin{table}[h!]
\small
 \caption{ Additional results for the MNIST dataset for the BaseVAE approach }
 \centering
 \begin{tabular}{|l| c |c |} 
 \hline
\textbf{Metric}  & {\textbf{Performance (mean $\pm$ std)}} \\ 
  \hline   
  IM1 Numerator & $3.37 \pm 0.21$ \\
  IM1 Denominator & $3.32 \pm 0.24$ \\
  IM2 Numerator & $1.02 \pm 0.09$ \\
  IM2 Denominator & $107.04 \pm 0.46$ \\
 \hline
 \end{tabular}
 \label{tab:mnist-add-res}
 \vspace{-0.5em}
\end{table}

\subsection{Implementation Details}
\label{sec:mnit-implement}
Th ML Model architecture and the BaseVAE architecture was kept the same as \cite{van2019interpretable}. The details can be seen in this \hyperlink{ https://github.com/SeldonIO/alibi/blob/master/examples/cfproto_mnist.ipynb }{notebook} by \cite{van2019interpretable}. We describe the implementation details of the ML model and BaseVAE below.

\subsubsection{ML Model Architecture}
Architecture of the ML model to be explained is now described. A slight difference we had to introduce from the architecture of \cite{van2019interpretable} was to make $kernel\textnormal{-}size$ as 3 ( instead of 2 ) and $padding$ as 1 to ensure the spatial dimensions of the image are the same as them after applying Conv Layer. The model was trained for 50 epochs, batch size 32, learning rate $10^{-4}$ with Adam optimizer and Cross Entropy Loss, with a 80-10-10\% training, validation and test dataset respectively. It achieved test accuracy of 96\%.

\begin{itemize}
    \item Conv Layer(out-channels=32, kernel-size=3, stride=1, padding=1 ), ReLU
    \item MaxPool(pool-size=2), Dropout(0.3)
    \item Conv Layer(out-channels=64, kernel-size=3, stride=1,     padding=1 ), ReLU
    \item MaxPool(pool-size=2), Dropout(0.3)
    \item Hidden Layer 1(256),  Dropout(0.5), ReLU
    \item Hidden Layer 2(10), Softmax
\end{itemize}

\subsubsection{Auto Encoder Architecture}
The training strategy used was exactly the same as mentioned in the \hyperlink{ https://github.com/SeldonIO/alibi/blob/master/examples/cfproto_mnist.ipynb }{notebook} by \cite{van2019interpretable}. 

Architecture of the Encoder used for computing the IM1, IM2 Metrics: 

\begin{itemize}
    \item Conv Layer(out-channel=16, kernel-size=3, stride=1, padding=0), ReLU
    \item Conv Layer(out-channel=16, kernel-size=3, stride=1, padding=0), ReLU   
    \item MaxPool(pool-size=2)
    \item Conv Layer(out-channel=1, kernel-size=3, stride=1, padding=0), ReLU   
\end{itemize}

Architecture of the Decoder used for computing the IM1, IM2 Metrics:
\begin{itemize}
    \item Conv Layer(out-channel=16, kernel-size=3, stride=1, padding=0), ReLU
    \item UpSample((2,2))
    \item Conv Layer(out-channel=16, kernel-size=3, stride=1, padding=0), ReLU   
    \item Conv Layer(out-channel=1, kernel-size=3, stride=1, padding=0)
\end{itemize}

\subsubsection{BaseGenCF Architecture}
The BaseVAE Encoder Decoder framework as trained for 25 epochs, batch size 16, learning rate $10^{-4}$ with SGD optimizer and \texttt{BaseGenCFLoss} (Section 3.1), with a 80-10-10\% training, validation and test dataset.

\textbf{Encoder Architecture}:

Architecture of Encoder consists of two networks: one is used to estimate the mean and the other is used to estimate the variance of the posterior distribution $q(z|x,y_{k})$. Both the networks for estimating mean and variance have the same architecture as described below, with the only difference that the variance network having an additional Sigmoid activation at the end to ensure the variance is positive. 

The network consists of sub models: the output from the sub model M1 is concatenated with the target class of the counterfactual and then fed into the sub model M2. We denote the size of the input after sub model M1 as \textit{convoluted-size} and the latent embedding space dimension as \textit{embedding-size}. We used \textit{embedding-size} as 10 and the \textit{convoluted-size} can be computed using the M1 architecture below to be $22*22$

Sub Model M1:
\begin{itemize}
    \item Conv Layer(out-channel=16, kernel-size=3, stride=1, padding=0), ReLU
    \item Conv Layer(out-channel=16, kernel-size=3, stride=1, padding=0), ReLU
    \item Conv Layer(out-channel=1, kernel-size=3, stride=1, padding=0), ReLU
\end{itemize}

Sub Model M2:
\begin{itemize}
    \item Hidden Layer( convoluted-size+1, embedding-size), BatchNorm
    \item Sigmoid (In case of variance network only)
\end{itemize}

\textbf{Decoder Architecture}:
\begin{itemize}
    \item Hidden Layer(embedding-size+1, convoluted-size), BatchNorm, ReLU
    \item Hidden Layer(convoluted-size, 2*convoluted-size), BatchNorm, ReLU
    \item Hidden Layer(2*convoluted-size, 28*28), BatchNorm, Sigmoid

\end{itemize}

\end{document}